%% file: main.tex
\documentclass[11pt]{article}

\usepackage[T1]{fontenc}
\usepackage{lmodern}
\usepackage[margin=1in]{geometry}
\usepackage{amsmath,amssymb,amsthm,mathtools}
\usepackage{natbib}
\usepackage{graphicx}
\usepackage{booktabs}
\usepackage{array}
\usepackage{placeins}
\usepackage{float}
\usepackage{xcolor}
\usepackage{xspace}
\usepackage{microtype}
\usepackage[hidelinks]{hyperref}
\hypersetup{
  hidelinks,
  hypertexnames=false,
  pdftitle={Posterior Information Dynamics of Diffusion Models for Linear Inverse Problems},
  pdfauthor={Xiangming Meng}
}

\newtheorem{theorem}{Theorem}[section]
\newtheorem{proposition}[theorem]{Proposition}
\newtheorem{lemma}[theorem]{Lemma}
\newtheorem{corollary}[theorem]{Corollary}
\theoremstyle{definition}

\theoremstyle{remark}
\newtheorem{remark}[theorem]{Remark}

\input{macros}
\input{paper_numbers}

\title{Posterior Information Dynamics of Diffusion Models for Linear Inverse Problems}

\author{Xiangming Meng\\
\small ZJU-UIUC Institute, Zhejiang University\\
\small Haining, Zhejiang 314400, China\\
\small \texttt{xiangmingmeng@intl.zju.edu.cn}}
\date{}

\begin{document}

\maketitle

\begin{abstract}%
Diffusion models are widely used as priors for linear inverse problems, yet
endpoint quality does not reveal when measurement information enters reverse
denoising or how it is allocated across signal directions.  We study this
process through the smoothed likelihood force, the difference between exact
posterior and prior scores at each noise level.  For a fixed measurement, its
expected squared norm gives both posterior--prior relative-entropy dissipation
and reverse-path relative-entropy growth.  Averaging over measurements yields
an information--minimum mean-square error (I-MMSE) identity linking information
gain to denoising-error reduction.  Under finite second moments, the force
energy and its ratio to prior-score energy decay quadratically in the noising
kernel's signal coefficient at high noise.  Solvable models show that
conditioning removes class separation already explained by the measurement,
reduces a uniform index entropy over \(n\) empirical samples from \(\log n\) to
\(H(I\mid r)\), and makes assimilation depend on operator--prior alignment even
for identical singular values.  Experiments in models with tractable
posteriors evaluate these predictions.  In a separate illustration with a
frozen FFHQ model, masks sharing the same spectrum yield different
prior-normalized null-space trajectory statistics.
\end{abstract}

\medskip
\noindent\textbf{Keywords:}
diffusion models, linear inverse problems, posterior information dynamics,
smoothed likelihood force, posterior sampling
\medskip

\section{Introduction}
\input{sections/intro}

\section{Background and Problem Setting}
\input{sections/background}

\section{Related Work}
\input{sections/related_work}

\section{Foundational Identities and Posterior Force Budgets}\label{sec:posterior-score}
\input{sections/posterior_score}

\section{Posterior Information Mechanisms in Solvable Models}\label{sec:solvable-reductions}
\input{sections/solvable_reductions}

\section{From Exact Information Dynamics to Approximate Samplers}\label{sec:regimes}
\input{sections/regimes}

\section{Diagnostic Studies}\label{sec:diagnostics}
\subsection{Exact-Model Diagnostics}\label{sec:exact-diagnostics}
\input{sections/experiments}
\FloatBarrier

\subsection{Neural Trajectory Diagnostics for Approximate Samplers}\label{sec:neural-diagnostics}
\input{sections/neural_trajectory}
\FloatBarrier

\section{Conclusion}
\input{sections/conclusion}

\section*{Acknowledgments and Disclosure of Funding}
The author thanks Yoshiyuki Kabashima for the valuable suggestion to analyze
solvable model cases.  This work was partly supported by the National Natural
Science Foundation of China (NSFC) under Grant No.~62306277.

\clearpage
\appendix
\section{Proofs}\label{app:proofs}
\input{appendix/proofs}

\FloatBarrier
\section{Experimental Details and Additional Analyses}\label{app:experiments}
\input{appendix/experiments}

\FloatBarrier
\bibliographystyle{plainnat}
\bibliography{refs}

\end{document}

%% file: macros.tex
\newcommand{\R}{\mathbb R}
\newcommand{\E}{\mathbb E}
\newcommand{\Prob}{\mathbb P}
\newcommand{\calN}{\mathcal N}
\newcommand{\dd}{\mathrm d}
\newcommand{\Id}{\mathrm I}
\newcommand{\Tr}{\operatorname{Tr}}
\newcommand{\Cov}{\operatorname{Cov}}
\newcommand{\Var}{\operatorname{Var}}
\newcommand{\diag}{\operatorname{diag}}
\newcommand{\sech}{\operatorname{sech}}
\newcommand{\KL}{\operatorname{KL}}
\newcommand{\MI}{\operatorname{I}}

\newcommand{\PiGDM}{PiGDM\xspace}

\newcommand{\runinhead}[1]{\par\noindent\emph{#1}\ }

%% file: paper_numbers.tex
\newcommand{\ProxyClosureMedian}{0.35\%}
\newcommand{\ProxyClosureNinetyFive}{3.21\%}
\newcommand{\HeuristicMAELogDet}{0.109}
\newcommand{\HeuristicMAEConstantLOSO}{0.133}

\newcommand{\HeuristicDiffConstantLOSO}{-0.024}
\newcommand{\HeuristicDiffConstantLOSOLo}{-0.064}
\newcommand{\HeuristicDiffConstantLOSOHi}{0.008}



%% file: sections/intro.tex
Diffusion and score-based models generate samples by reversing a stochastic
noising process
\citep{sohl2015deep,song2019generative,ho2020denoising,song2021score}.
The reverse trajectory is not simply a route from noise to a final sample.
Recent exact-score analyses identify Brownian, speciation, and collapse
regimes in which coarse and then fine uncertainty is progressively resolved
\citep{biroli2024dynamical,raya2023spontaneous,sclocchi2025phase,
li2024critical,ambrogioni2025thermodynamics}.  This dynamical picture raises
the question studied here: when conditioning on a linear inverse measurement,
what uncertainty is removed by the measurement, what remains to be resolved
during reverse denoising, and how is the remaining uncertainty distributed
across signal directions?

A linear inverse problem seeks to recover an unknown signal
\(x_0\in\mathbb R^d\) from an indirect, noisy measurement
\begin{equation}
\label{eq:intro-inverse-measurement}
r=Ax_0+\varepsilon,
\end{equation}
where \(A\) is a known linear operator, \(r\) is the measurement, and
\(\varepsilon\) is measurement noise.  Rank deficiency, poor conditioning, or
noise generally leaves multiple signals compatible with the same measurement,
so the inferential target is a posterior distribution rather than a unique
inverse.  Choosing \(A\) as an identity, masking, convolution, or downsampling
operator recovers denoising, inpainting, deblurring, and super-resolution,
respectively.

Diffusion priors have been incorporated into inverse solvers through
score-based and likelihood-guided posterior sampling
\citep{song2021medical,jalal2021robust,chung2023diffusion,
meng2022diffusion}, operator-aware spectral and null-space constructions
\citep{kawar2022denoising,wang2023zero}, pseudoinverse guidance
\citep{song2023pigdm}, and variational or plug-and-play formulations
\citep{mardani2023variational,zhu2023diffpir}.  These methods are commonly
evaluated through final reconstruction quality.  Such endpoint metrics do not
characterize the conditional trajectory: they cannot distinguish uncertainty
removed directly by the measurement from uncertainty resolved later by
denoising, nor can they show how this division varies across signal
directions.  Nor can the unconditional dynamical picture be transferred
without modification.  Conditioning can establish a class preference before
the noisy state supplies additional evidence, reweight the plausible elements
of an empirical prior, and distinguish strongly observed directions from
weakly observed or null directions.  A conditional theory must therefore
separate information supplied directly by the measurement from information
resolved during denoising, while retaining the geometry induced by \(A\).

The main analytical obstacle is a mismatch between the likelihood and the
evolving state: \(p(r\mid x_0)\) is defined at the clean signal \(x_0\), whereas
reverse diffusion evolves the noisy state \(X_t=x\).  At noise level \(t\), the
exact contribution of the measurement to the posterior score is
\begin{equation}
\label{eq:intro-smoothed-likelihood-force}
g_t(x,r)
=\nabla_x\log p_t(x\mid r)-\nabla_x\log p_t(x)
=\nabla_x\log p(r\mid X_t=x),
\end{equation}
where \(p_t\) and \(p_t(\cdot\mid r)\) are the noised prior and posterior
marginals.  We call \(g_t\) the \emph{smoothed likelihood force}.  Unlike a
clean-space residual evaluated at a denoised estimate, it averages the
measurement likelihood over the remaining uncertainty in \(X_0\) given
\(X_t=x\).  For a general prior, evaluating \(g_t\) requires the conditional law
of \(X_0\) given \(X_t=x\), so a model-independent closed form is not generally
available.  Recent work estimates posterior scores or analyzes particular
posterior-sampling paths
\citep{mammadov2026exactposterior,delgadino2026feynman,
crafts2026stabilized}.  Our analysis is complementary: it derives the exact
information law obeyed by \(g_t\) before introducing a sampling approximation.

The key observation is that conditioning changes the initial law, but not the
forward noising mechanism.  For each fixed \(r\), \(p_t(\cdot\mid r)\) and
\(p_t\) therefore evolve under the same diffusion semigroup, making \(g_t\)
their relative score.  Although this vector field need not be available
pointwise, its energy admits exact characterizations.  A relative de Bruijn
identity equates that energy with the forward-time dissipation of
\(\mathrm{KL}(p_t(\cdot\mid r)\Vert p_t)\)
\citep{foellmer1985entropy}.  Exact time reversal and finite-entropy Girsanov
theory identify its accumulated energy, together with the terminal marginal
divergence, with reverse-path relative entropy
\citep{anderson1982reverse,haussmann1986time,leonard2012girsanov}.
Measurement averaging yields an information--minimum mean-square error
(I-MMSE) relation between reverse-time information gain and denoising-error
reduction \citep{guo2005mutual}, while a conditional Tweedie identity expresses
\(g_t\) as the scaled difference between measurement-conditioned and
unconditional denoisers \citep{louis1982finding,efron2011tweedie}.  These
identities answer the temporal part of the central question by quantifying how
much measurement information is present and when it is acquired.  The solvable
models then describe how the remaining uncertainty is distributed among
classes, empirical explanations, and operator-aligned directions.

The main contributions are as follows.

\begin{enumerate}
\item We derive a fixed-measurement information budget whose instantaneous
rate is the conditional force energy.  Its integral, together with the
terminal marginal divergence, determines exact posterior--prior reverse-path
separation.  Measurement averaging yields the corresponding I-MMSE
representation.  Under the joint law of the noisy state and measurement, we
further prove, assuming finite second moments, that the force energy and its
ratio to prior-score energy decay quadratically in the noising kernel's signal
coefficient at high noise.

\item We identify three ways in which conditioning changes the uncertainty
resolved along a diffusion trajectory.  In a Gaussian mixture, only class
separation left unexplained by the measurement can generate new class
commitment.  For a uniform empirical prior supported on \(n\) samples, with
\(I\) denoting the sampled support index, the measurement reduces the remaining
explanation budget from \(H(I)=\log n\) to the conditional Shannon entropy
\(H(I\mid r)\).  In an anisotropic Gaussian model, equal operator singular
spectra can nevertheless produce different information gain, posterior
contraction, and reverse-path separation because of operator--prior alignment.

\item We evaluate the exact quantities in finite-mixture, empirical-prior, and
Gaussian models, and use the resulting timing and subspace coordinates to
study approximate samplers driven by a frozen diffusion model trained on the
Flickr-Faces-HQ (FFHQ) face-image dataset \citep{karras2019style}.  With a
fixed checkpoint and a fixed-spectrum mask library, this learned-model
illustration examines how an empirical null-space trajectory statistic varies
with mask geometry; it is not a direct test of the Gaussian extremal result.
\end{enumerate}

Section~2 introduces the probabilistic setting, notation, and forward noising
process.  Section~3 reviews related work.  Section~4 develops the posterior
force and its information identities.  Section~5 analyzes residual class
information, empirical explanation resolution, and directional allocation in
solvable models.  Section~6 synthesizes these objects and connects them to
observables available for approximate samplers.  Section~7 presents the
exact-model and neural trajectory studies.  Section~8 concludes.  Proofs and
supporting experimental details are collected in the appendices.

%% file: sections/background.tex
\label{sec:background}

\runinhead{Notation and forward noising.}
Random variables are denoted by capital letters and realized values by lowercase
letters.  The clean signal is \(X_0\in\R^d\), a fixed clean realization is
written \(x_0\), the random measurement is \(R\), and a realized measurement is
written \(r\).  We use \(P_0\) for the clean-data law and \(p_0\) for its
density when one exists.

The analysis uses the variance-preserving Ornstein--Uhlenbeck noising process
\citep{uhlenbeck1930theory,ho2020denoising,song2021score},
\begin{equation}
\label{eq:forward-marginal}
X_t=a_tX_0+\sqrt{\Delta_t}\xi,\qquad
a_t=e^{-t},\qquad
\Delta_t=1-e^{-2t},\qquad
\xi\sim\calN(0,\Id_d).
\end{equation}
Its continuous-time realization is
\begin{equation}
\label{eq:forward-ou-sde}
\dd X_t=-X_t\,\dd t+\sqrt{2}\,\dd W_t.
\end{equation}
Here \(W\) is a standard \(d\)-dimensional Brownian motion.
The transition density is
\begin{equation}
\label{eq:forward-kernel}
q_t(x\mid x_0)
=(2\pi\Delta_t)^{-d/2}
\exp\left\{-\frac{\|x-a_tx_0\|^2}{2\Delta_t}\right\},
\end{equation}
and, for every \(t>0\), the noised prior has density
\begin{equation}
\label{eq:noised-distribution}
p_t(x)=\E\!\left[q_t(x\mid X_0)\right].
\end{equation}
When \(P_0\) has density \(p_0\), the expectation in
\eqref{eq:noised-distribution} is
\(\int q_t(x\mid x_0)p_0(x_0)\,\dd x_0\); for an empirical prior it is a
finite sum.  The unconditional score is
\begin{equation}
\label{eq:prior-score-definition}
s_t^{\rm prior}(x)=\nabla_x\log p_t(x).
\end{equation}
For a terminal time \(T\) and increasing reverse clock \(u\in[0,T]\), write
\(Z_u=X_{T-u}\).  Whenever the standard time-reversal conditions hold
\citep{anderson1982reverse,haussmann1986time}, the prior reverse process
satisfies
\begin{equation}
\label{eq:prior-reverse-sde}
\dd Z_u=
\bigl\{Z_u+2s_{T-u}^{\rm prior}(Z_u)\bigr\}\,\dd u
+\sqrt{2}\,\dd\overline W_u.
\end{equation}
Here \(\overline W\) is Brownian motion in the reverse filtration.
The question studied here is how this reverse drift changes once the clean law
is conditioned on a measurement.

\runinhead{Measurements and posterior dynamics.}
Let \((\mathsf R,\nu)\) be a measurement space equipped with a sigma-finite
dominating measure.  We write \(p(r\mid x_0)\) for a density (or probability
mass function) with respect to \(\nu\).  Our standing assumption is
\begin{equation}
\label{eq:indep-assump}
R\perp X_t\mid X_0,
\end{equation}
equivalently, the measurement and forward noise are conditionally independent
given the clean signal.  This assumption allows an arbitrary clean-space
likelihood and is the factorization used in
Section~\ref{sec:posterior-score}.

For linear Gaussian inverse problems, \(\mathsf R=\R^{d_y}\) and
\[
R=AX_0+\varepsilon,
\qquad
A\in\R^{d_y\times d},
\qquad
\varepsilon\sim\calN(0,\Gamma),
\qquad
\Gamma\succ0,
\]
with \(\varepsilon\) independent of \(X_0\) and the forward noise.  We write
\(r_A=\operatorname{rank}(A)\).  The marginal measurement density is
\[
p_R(r)=\E\!\left[p(r\mid X_0)\right].
\]
Fix \(r\) with \(0<p_R(r)<\infty\), and let
\(P_0(\,\cdot\mid r)\) denote the clean posterior law.  Its noised density is
\begin{equation}
\label{eq:noised-posterior-law}
p_t(x\mid r)=\E\!\left[q_t(x\mid X_0)\mid R=r\right].
\end{equation}
The exact posterior score is
\begin{equation}
\label{eq:posterior-score-definition}
s_t^r(x)=\nabla_x\log p_t(x\mid r).
\end{equation}
The corresponding reverse process is
\begin{equation}
\label{eq:posterior-reverse-sde}
\dd Z_u^r=
\bigl\{Z_u^r+2s_{T-u}^r(Z_u^r)\bigr\}\,\dd u
+\sqrt{2}\,\dd\overline W_u^r,
\end{equation}
initialized from \(p_T(\cdot\mid r)\), with \(\overline W^r\) the corresponding
reverse-time Brownian motion.
Section~\ref{sec:posterior-score} identifies
\(s_t^r-s_t^{\rm prior}\) with the smoothed likelihood force.  Under the
normalization \eqref{eq:forward-ou-sde}, twice this force is exactly the
difference between the posterior and prior reverse drifts.

\runinhead{Operator geometry.}
The linear operator separates measured and unmeasured directions.  We write
\(A^\top\) for the Euclidean adjoint and \(A^\dagger\) for the Moore--Penrose
pseudo-inverse.  When this distinction is relevant, we use
\begin{equation}
\label{eq:row-null-projectors}
P_A=A^\dagger A,\qquad P_A^\perp=\Id_d-P_A
\end{equation}
as the Euclidean row-space and null-space projectors.  These projectors provide
diagnostic coordinates for the later subspace analysis.  Outside the isotropic
linear-Gaussian reference, the exact force need not be row supported because
prior dependence can transmit measurement influence across Euclidean
coordinates.

\runinhead{Empirical priors and index uncertainty.}
For an empirical prior
\begin{equation}
\label{eq:empirical-prior}
P_0^{\rm emp}=\frac1n\sum_{i=1}^n\delta_{x_i},
\end{equation}
one may equivalently sample an index \(I\) uniformly from
\(\{1,\ldots,n\}\) and set \(X_0=x_I\).  In the exact-score unconditional
theory, the reverse dynamics exhibit high-noise Brownian behavior,
intermediate commitment to coarse structure, and low-noise collapse around
individual training samples \citep{biroli2024dynamical}.  The entropy of the
uniform index is \(\log n\).

For a fixed measurement \(r\), conditioning assigns posterior weights
\begin{equation}
\label{eq:background-posterior-weights}
w_i(r)=\Prob(I=i\mid R=r),
\end{equation}
whose Shannon entropy is
\(H(I\mid r)=-\sum_i w_i(r)\log w_i(r)\).  The state-resolved information
\(\MI(I;X_t\mid R=r)\) measures how much of this conditional index uncertainty
has been resolved at noise level \(t\).  Section~\ref{sec:collapse} develops
this decomposition, and Section~\ref{sec:regime-map-synthesis} relates it to
the measurement--state information of
Proposition~\ref{prop:force-information-rate}.

%% file: sections/related_work.tex
\runinhead{Diffusion dynamics and information identities.}
Diffusion models generate samples by reversing a stochastic noising process
\citep{sohl2015deep,song2019generative,ho2020denoising,song2021score}.
Denoising score matching connects denoising to score estimation
\citep{vincent2011connection}.
Exact-score analyses interpret this reversal as progressive uncertainty
resolution and identify Brownian, speciation, and empirical-collapse behavior
in unconditional models
\citep{biroli2024dynamical,raya2023spontaneous,sclocchi2025phase,
li2024critical,ambrogioni2025thermodynamics}.  Extensions address discrete
state spaces \citep{takahashi2026discrete}, multimodal and more general class
structure \citep{albrychiewicz2026multimodal,achilli2026speciation}, and
controlled or trained-model dynamics
\citep{lu2026steering,handke2026entropic}.
The information identities used below have complementary origins.  Relative
entropy along a common diffusion semigroup dissipates at the relative
Fisher-information rate \citep{foellmer1985entropy}, while time reversal gives
the associated path-measure relations
\citep{anderson1982reverse,haussmann1986time,leonard2012girsanov}.
The I-MMSE identity connects Gaussian-channel information growth to
conditional mean-square error \citep{guo2005mutual} and has been applied to
diffusion likelihoods \citep{kong2023information}.  MMG estimates mutual
information by integrating the gap between conditional and unconditional
denoising errors along a diffusion channel \citep{yu2026mmg}; score-difference
path energies likewise yield KL and mutual-information estimators
\citep{franzese2024minde}.  Fisher and Tweedie identities connect scores to
conditional means \citep{louis1982finding,efron2011tweedie}.  Here these
established identities are coupled at a fixed inverse-problem measurement to
connect force energy, marginal KL dissipation, reverse-path KL, and empirical
information resolution before a sampler approximation is chosen.

\runinhead{Posterior scores and conditional samplers.}
Pretrained diffusion priors support likelihood-guided samplers
\citep{song2021medical,jalal2021robust,chung2023diffusion}.
Likelihood-score approximations address Gaussian and quantized measurements
\citep{meng2023quantized,meng2024qcssgmplus,meng2022diffusion}.  Other
approaches include operator-aware spectral, null-space, and pseudoinverse methods
\citep{kawar2022denoising,wang2023zero,song2023pigdm}; posterior-moment
approximations \citep{boys2024tweedie}; denoiser-based methods
\citep{kadkhodaie2021solving}; and variational or plug-and-play formulations
\citep{mardani2023variational,zhu2023diffpir}.
\citet{mammadov2026exactposterior} derived a closed-form posterior-score
representation for linear Gaussian inverse problems and used it as an Exact
Posterior Score (EPS) training objective.  \citet{jiao2026pddim} instead
developed a coordinate-wise sampler that switches between prior and measurement
predictors according to singular-direction signal-to-noise ratios, using DDIM
updates \citep{song2021ddim}.  Particle methods provide another route: TDS
retains asymptotic exactness through sequential Monte Carlo
\citep{wu2023practical}, and MCGdiff samples intermediate linear-inverse
posteriors \citep{cardoso2024monte}.  Related work estimates denoising-posterior
covariance \citep{peng2024improving} or uses conditional mutual information as
an algorithmic objective \citep{hamidi2025conditional}.

Trajectory-level analyses address the error of approximate samplers.
\citet{delgadino2026feynman} characterize accumulated bias and stability of
DPS-like paths through a Feynman--Kac representation, while
\citet{crafts2026stabilized} study likelihood-weighted path-space control.
These works construct conditional samplers or compare approximate controlled
paths with a target.  Our complementary question precedes that choice: what
information law does the exact posterior--prior score difference obey?  In
particular, our path identity compares the exact prior and posterior reverse
laws generated by the same forward semigroup, rather than the bias of an
approximate path.

\runinhead{Empirical priors, conditional entropy, and memorization.}
For empirical or codebook priors, unconditional collapse is governed by the
information needed to identify a support element \citep{biroli2024dynamical}.
\citet{hunt2026bird} characterize a memorization--generalization boundary by
comparing the information in restricted noisy observations with \(\log n\),
where \(n\) is the number of training samples.
\citet{nguyen2026lemmse} derive a Local-Equivariant MMSE approximation for
trained convolutional inverse solvers under empirical distributions.  For our
fixed-measurement empirical posterior, the measurement supplies
\(\log n-H(I\mid r)\) nats of index information and leaves \(H(I\mid r)\) to
reverse denoising.  This connects to memorization and extraction
\citep{carlini2023extracting,somepalli2023diffusion}, although our formal
results remain restricted to explicit linear likelihoods and calibrated
posterior weights.

\runinhead{Operator geometry.}
Classical Gaussian inverse problems make the interaction between prior
covariance and measurement subspaces explicit.  Diffusion restoration methods
exploit the same singular-vector and row/null geometry
\citep{kawar2022denoising,wang2023zero,song2023pigdm}.  Our matched-spectrum
result fixes the singular values and varies only their orientation relative to
an anisotropic prior.  At the clean endpoint, selecting the principal prior
subspace is the classical Gaussian D-optimal design problem
\citep{chaloner1995bayesian}; our result extends the same orientation ordering
to every finite Gaussian-channel SNR and hence to the reverse-path information
budget.  The neural mask experiment asks only whether singular spectra suffice
for a frozen learned trajectory statistic, rather than testing the Gaussian
extremal formula.

%% file: sections/posterior_score.tex
\subsection{Posterior Score and Denoiser Identities}
\label{sec:foundational-identities}

Under the conditional-independence
assumption~\eqref{eq:indep-assump}, Bayes' rule gives
\begin{equation}
\label{eq:posterior-bayes-factorization}
p_t(x\mid r)
=
\frac{p_t(x)p(r\mid X_t=x)}{p_R(r)},
\qquad
p(r\mid X_t=x)
=
\E\!\left[p(r\mid X_0)\mid X_t=x\right],
\end{equation}
whenever the relevant densities exist.  Differentiating in \(x\) yields
\begin{equation}
\label{eq:posterior-score-factorization}
\nabla_x\log p_t(x\mid r)
=
\nabla_x\log p_t(x)+\nabla_x\log p(r\mid X_t=x).
\end{equation}
The second term,
\[
g_t(x,r):=\nabla_x\log p(r\mid X_t=x),
\]
is the smoothed likelihood force introduced in the Introduction.
It averages the clean-space likelihood over \(X_0\mid X_t=x\), rather than
evaluating that likelihood at a point estimate of the clean signal.

Fisher's latent-variable score identity
\(\nabla_y\log\int p(y,z)\,\dd z
=\E[\nabla_y\log p(y,Z)\mid y]\)
\citep{louis1982finding} turns this force into a difference of conditional
denoisers.

\begin{lemma}[Conditional Tweedie identity]
\label{thm:posterior-tweedie}
Suppose that \(\E\|X_0\|^2<\infty\) and that differentiation may be
interchanged with the noising integrals in
\eqref{eq:noised-distribution} and \eqref{eq:noised-posterior-law}.  For
\(t>0\), define
\[
m_t(x)=\E[X_0\mid X_t=x],
\qquad
m_t^r(x)=\E[X_0\mid X_t=x,R=r].
\]
Then, for every \(r\) with \(0<p_R(r)<\infty\),
\begin{equation}
\label{eq:posterior-tweedie-score}
\nabla_x\log p_t(x\mid r)=\frac{a_tm_t^r(x)-x}{\Delta_t},
\end{equation}
and the smoothed likelihood force is
\begin{equation}
\label{eq:general-denoiser-guidance}
g_t(x,r):=\nabla_x\log p(r\mid X_t=x)
=
\frac{a_t}{\Delta_t}\{m_t^r(x)-m_t(x)\}.
\end{equation}
\end{lemma}

This conditional form of Tweedie's formula is standard
\citep{efron2011tweedie,boys2024tweedie,mammadov2026exactposterior}.  Here it
provides the bridge from the exact posterior score to the force-energy and
information identities below.

\subsection{Fixed-Measurement Posterior Information Budget}
\label{sec:fixed-measurement-budget}

The score difference in Lemma~\ref{thm:posterior-tweedie} is also the relative
score of two laws propagated by the same Ornstein--Uhlenbeck semigroup.  Its
energy therefore controls the loss of posterior--prior distinguishability at
each fixed measurement.

\begin{theorem}[Fixed-measurement posterior information budget]
\label{thm:fixed-r-entropy-dissipation}
Fix a measurement value $r$ with $0<p_R(r)<\infty$ and define
\[
q_t^r(x):=p_t(x\mid r),
\qquad
\mathcal K_t(r):=\KL(q_t^r\|p_t).
\]
Let $0<s<t<\infty$.  Assume $\mathcal K_s(r)<\infty$, the relative Fisher
information
\[
\int q_u^r(x)
\left\|\nabla\log\frac{q_u^r(x)}{p_u(x)}\right\|^2\dd x
\]
is integrable over $u\in[s,t]$, and the integration-by-parts boundary terms
vanish.  Then
\begin{equation}
\label{eq:fixed-r-kl-integrated}
\mathcal K_s(r)-\mathcal K_t(r)
=
\int_s^t
\E_{X_u\mid r}\|g_u(X_u,r)\|^2\,\dd u.
\end{equation}
Consequently, for almost every $t>0$ (and at every differentiability point),
\begin{equation}
\label{eq:fixed-r-kl-dissipation}
-\frac{\dd}{\dd t}\mathcal K_t(r)
=
\E_{X_t\mid r}\|g_t(X_t,r)\|^2.
\end{equation}
Averaging over the measurement gives
\begin{equation}
\label{eq:joint-mutual-information-dissipation}
\MI(R;X_s)-\MI(R;X_t)
=
\int_s^t\E\|g_u(X_u,R)\|^2\,\dd u,
\qquad
-\frac{\dd}{\dd t}\MI(R;X_t)
=
\E\|g_t(X_t,R)\|^2
\end{equation}
whenever the derivative exists.  Thus force energy is the forward-time
dissipation rate of posterior--prior relative entropy, or the corresponding
reverse-time assimilation rate.
\end{theorem}

This is the relative de Bruijn identity for a common diffusion semigroup
\citep{foellmer1985entropy}.  Equation~\eqref{eq:fixed-r-kl-dissipation}
uses diffusion time; Proposition~\ref{prop:force-information-rate} below
reparameterizes its measurement average by Gaussian-channel signal-to-noise
ratio.

\begin{lemma}[Regularity in the solvable model classes]
\label{lem:budget-regularity}
The regularity conditions of
Theorem~\ref{thm:fixed-r-entropy-dissipation}, together with the endpoint
limits used below, hold in either of the following settings.
\begin{enumerate}
\item \(P_0=\sum_{i=1}^n\pi_i\delta_{x_i}\) has finite support with
\(\pi_i>0\), and \(P_0(\cdot\mid r)=\sum_iw_i(r)\delta_{x_i}\).
If the support points are pairwise distinct, then
\[
\lim_{t\to0^+}\mathcal K_t(r)
=\sum_{i=1}^n w_i(r)\log\frac{w_i(r)}{\pi_i}.
\]
\item \(P_0=\calN(\mu,\Sigma)\) and
\(P_0(\cdot\mid r)=\calN(\mu_r,S_r)\), with
\(\Sigma\succ0\) and \(S_r\succ0\).
\end{enumerate}
In both cases \(\lim_{t\to\infty}\mathcal K_t(r)=0\).  On every compact
interval \(0<s<t<\infty\), the reverse path laws below have finite relative
entropy.
\end{lemma}

\begin{corollary}[Exact posterior--prior reverse-path relative entropy]
\label{cor:reverse-path-kl}
Fix $T>t>0$.  On the reverse clock $u\in[0,T-t]$, let
$\mathbb Q^{r,\mathrm{rev}}_{[t,T]}$ and
$\mathbb P^{\mathrm{rev}}_{[t,T]}$ be the time-reversal pushforwards of the
forward OU path laws initialized from $q_t^r$ and $p_t$, respectively.
When the noised marginals are smooth and positive, they coincide with the
reverse processes \eqref{eq:posterior-reverse-sde} and
\eqref{eq:prior-reverse-sde}, initialized from $q_T^r$ and $p_T$.
Assume the hypotheses of
Theorem~\ref{thm:fixed-r-entropy-dissipation} on \([t,T]\).  Then
\begin{equation}
\label{eq:reverse-path-kl-identity}
\KL\!\left(
\mathbb Q^{r,\mathrm{rev}}_{[t,T]}
\middle\|
\mathbb P^{\mathrm{rev}}_{[t,T]}
\right)
=
\mathcal K_t(r)
=
\mathcal K_T(r)+
\int_t^T\E_{X_u\mid r}\|g_u(X_u,r)\|^2\,\dd u.
\end{equation}
Under the normalization $\dd X_t=-X_t\dd t+\sqrt2\dd W_t$, the coefficient of
the integral is one: the two reverse drifts differ by $2g_t$.
\end{corollary}

The path identity uses the exact time-\(T\) prior and posterior marginals.
Initializing both reverse processes from a common approximate Gaussian adds a
separate terminal mismatch.

\begin{corollary}[High-noise reverse-path proximity]
\label{cor:high-noise-path-proximity}
Under Corollary~\ref{cor:reverse-path-kl}, assume
$\Tr\Cov(X_0)<\infty$, and set
\[
\gamma_t=\frac{a_t^2}{\Delta_t},
\qquad
Y_{\gamma_t}=\frac{X_t}{\sqrt{\Delta_t}}.
\]
Then, for every $T>t>0$,
\begin{equation}
\label{eq:average-path-kl-bound}
\begin{aligned}
\E_R\KL\!\left(
\mathbb Q^{R,\mathrm{rev}}_{[t,T]}
\middle\|
\mathbb P^{\mathrm{rev}}_{[t,T]}
\right)
&=\MI(R;X_t)=\MI(R;Y_{\gamma_t})\\
&\le \MI(X_0;Y_{\gamma_t})\\
&\le \frac12\log\det\!\left(\Id_d+\gamma_t\Cov(X_0)\right)\\
&\le \frac{\gamma_t}{2}\Tr\Cov(X_0).
\end{aligned}
\end{equation}
Consequently, if $d^{-1}\Tr\Cov(X_0)\le C_0$ and
$\Delta_t\ge\Delta_0>0$, the average reverse-path KL per coordinate is at
most $C_0a_t^2/(2\Delta_0)$.  Pinsker's inequality and Jensen's inequality also
give
\begin{equation}
\label{eq:average-path-tv-bound}
\E_R\left\|
\mathbb Q^{R,\mathrm{rev}}_{[t,T]}
-
\mathbb P^{\mathrm{rev}}_{[t,T]}
\right\|_{\rm TV}
\le
\sqrt{\frac{\gamma_t}{4}\Tr\Cov(X_0)}.
\end{equation}
Here
\(\|\mu-\nu\|_{\rm TV}:=\sup_B|\mu(B)-\nu(B)|
=\tfrac12\int|\dd\mu-\dd\nu|\).
The KL bound is naturally normalized by dimension.  The total-variation bound
vanishes in fixed dimension as \(a_t\to0\), and along dimension sequences for
which \(d a_t^2\to0\) on the high-noise window.
\end{corollary}

\begin{corollary}[Integrated head start]
\label{cor:integrated-head-start}
Suppose that $P_0(\cdot\mid r)\ll P_0$,
$\KL(P_0(\cdot\mid r)\|P_0)<\infty$, and the endpoint limits of
Theorem~\ref{thm:fixed-r-entropy-dissipation} hold.  Then
\begin{equation}
\label{eq:integrated-force-clean-kl}
\int_0^\infty
\E_{X_t\mid r}\|g_t(X_t,r)\|^2\,\dd t
=
\KL\!\left(P_0(\cdot\mid r)\middle\|P_0\right).
\end{equation}
For the uniform empirical prior on pairwise distinct points,
\begin{equation}
\label{eq:integrated-force-head-start}
\int_0^\infty
\E_{X_t\mid r}\|g_t(X_t,r)\|^2\,\dd t
=
\sum_{i=1}^n w_i(r)\log\{n w_i(r)\}
=
\log n-H(I\mid r).
\end{equation}
At fixed \(r\), the information supplied by the measurement before reverse
denoising is therefore the accumulated force energy along the full noising
path.
\end{corollary}

Lemma~\ref{lem:budget-regularity} verifies these conditions for the finite
empirical and linear-Gaussian model classes used below.

\subsection{Joint-Law Force Energy and High-Noise Scaling}
\label{sec:high-noise-dominance}

Averaging over the random measurement yields an additional orthogonality of
the prior score and the measurement force.

\begin{lemma}[Joint score orthogonality and Fisher-energy decomposition]
\label{prop:joint-fisher-pythagoras}
Under the joint law of $(X_t,R)$, assume the relevant squared scores are
integrable and that differentiation may be passed under the \(\nu\)-integral
in \(p_t(x)=\int p_t(x,r)\,\dd\nu(r)\).  Write
$s_t^R(X_t)=s_t^r(X_t)|_{r=R}$.  Then
\begin{equation}
\label{eq:force-conditional-mean-zero}
\E[g_t(X_t,R)\mid X_t]=0,
\end{equation}
and therefore the score second-moment matrices obey
\begin{equation}
\label{eq:matrix-fisher-pythagoras}
\E\!\left[s_t^R(X_t)s_t^R(X_t)^\top\right]
=
\E\!\left[s_t^{\rm prior}(X_t)s_t^{\rm prior}(X_t)^\top\right]
+
\E\!\left[g_t(X_t,R)g_t(X_t,R)^\top\right].
\end{equation}
Taking the trace gives
\begin{equation}
\label{eq:joint-fisher-pythagoras}
\E\|s_t^R(X_t)\|^2
=
\E\|s_t^{\rm prior}(X_t)\|^2
+
\E\|g_t(X_t,R)\|^2,
\end{equation}
and, for every deterministic vector $v$,
\begin{equation}
\label{eq:directional-fisher-pythagoras}
\E|v^\top s_t^R(X_t)|^2
=
\E|v^\top s_t^{\rm prior}(X_t)|^2
+
\E|v^\top g_t(X_t,R)|^2.
\end{equation}
\end{lemma}

\begin{corollary}[High-noise decay of the smoothed likelihood force]
\label{cor:general-high-noise-energy}
Under the assumptions of Lemma~\ref{thm:posterior-tweedie}, the force
\(g_t\) satisfies, under the joint law of \((X_t,R)\),
\begin{equation}
\label{eq:general-guidance-energy-identity}
\E\|g_t(X_t,R)\|^2
=
\frac{a_t^2}{\Delta_t^2}
\E\big\|\E[X_0\mid X_t,R]-\E[X_0\mid X_t]\big\|^2.
\end{equation}
Moreover,
\begin{equation}
\label{eq:general-guidance-energy-bound}
\E\|g_t(X_t,R)\|^2
\le
\frac{a_t^2}{\Delta_t^2}\Tr\Cov(X_0).
\end{equation}
Consequently, along any dimension sequence for which
\(d^{-1}\Tr\Cov(X_0)=O(1)\), and on any high-noise window where
\(\Delta_t\ge \Delta_0>0\),
\begin{equation}
\label{eq:general-high-noise-scaling}
\frac1d\E\|g_t(X_t,R)\|^2=O(a_t^2).
\end{equation}
\end{corollary}

The scaling is quadratic in the signal coefficient \(a_t\), under a finite
per-coordinate second moment.  The condition
\(\Delta_t\ge\Delta_0>0\) specifies the high-noise window and prevents the
factor \(\Delta_t^{-2}\) from diverging.

\begin{lemma}[Automatic prior-score energy lower bound]
\label{lem:prior-score-lower-bound}
Assume $\Tr\Cov(X_0)<\infty$ and let $\mu_t=\E X_t$.  For every $t>0$,
\begin{equation}
\label{eq:automatic-prior-score-lower}
\frac1d\E\|s_t^{\rm prior}(X_t)\|^2
\ge
\frac{d}{\Tr\Cov(X_t)}
=
\frac{d}{a_t^2\Tr\Cov(X_0)+d\Delta_t}.
\end{equation}
In particular, if $d^{-1}\Tr\Cov(X_0)\le C_0$, then
\begin{equation}
\label{eq:automatic-prior-score-lower-C0}
\frac1d\E\|s_t^{\rm prior}(X_t)\|^2
\ge
\frac1{a_t^2C_0+\Delta_t}.
\end{equation}
\end{lemma}

\begin{corollary}[Likelihood-to-prior force ratio]
\label{cor:general-relative-energy}
Assume the conditions of Corollary~\ref{cor:general-high-noise-energy} and
$d^{-1}\Tr\Cov(X_0)\le C_0$.  Then, for every $t>0$,
\begin{equation}
\label{eq:general-relative-energy-ratio}
\frac{\E\|g_t(X_t,R)\|^2}
{\E\|s_t^{\rm prior}(X_t)\|^2}
\le
\frac{C_0(a_t^2C_0+\Delta_t)}{\Delta_t^2}\,a_t^2.
\end{equation}
Consequently, on every high-noise window where $\Delta_t\ge\Delta_0>0$,
\begin{equation}
\label{eq:general-relative-energy-ratio-window}
\frac{\E\|g_t(X_t,R)\|^2}
{\E\|s_t^{\rm prior}(X_t)\|^2}
\le
\frac{C_0(C_0+1)}{\Delta_0^2}\,a_t^2.
\end{equation}
\end{corollary}

The lower bound makes a separate prior-score nondegeneracy assumption
unnecessary.  Thus, under the same second-moment condition, the exact force is
small both absolutely and relative to the prior score at high noise.  For
bounded empirical priors, Appendix~A gives the sharper asymptotic
$d^{-1}\E\|s_t^{\rm prior}(X_t)\|^2=\Delta_t^{-1}+o(1)$.

\subsection{Measurement Assimilation and I-MMSE}
\label{sec:assimilation-flow}

Rescaling the noised state gives the canonical Gaussian channel
\begin{equation}
\label{eq:canonical-gaussian-channel}
Y_\gamma=\sqrt{\gamma}\,X_0+Z,
\qquad
Z\sim\calN(0,\Id_d),
\qquad
Y_{\gamma_t}=\frac{X_t}{\sqrt{\Delta_t}},
\qquad
\gamma_t=\frac{a_t^2}{\Delta_t}.
\end{equation}
Here \(Z\) is independent of \((X_0,R)\).
Because \(X_t\mapsto Y_{\gamma_t}\) is invertible,
\(\MI(R;Y_{\gamma_t})=\MI(R;X_t)\).  Moreover,
\(R\perp Y_\gamma\mid X_0\).  Define the measurement-assimilation information
at channel SNR \(\gamma\) by
\begin{equation}
\label{eq:measurement-assimilation-information}
\mathcal M_\gamma:=\MI(R;Y_\gamma),
\end{equation}
which measures the information shared by the measurement and the current noisy
state.  Conditional I-MMSE \citep{guo2005mutual} identifies its derivative
with the denoiser difference in
\eqref{eq:general-denoiser-guidance}.

\begin{proposition}[I-MMSE representation of measurement-force energy]
\label{prop:force-information-rate}
Under the assumptions of Lemma~\ref{thm:posterior-tweedie}, let
$\widehat m_\gamma(Y_\gamma)=\E[X_0\mid Y_\gamma]$ and
$\widehat m_\gamma^R(Y_\gamma,R)=\E[X_0\mid Y_\gamma,R]$.  Then, for
$\gamma>0$ (and by right differentiation at $\gamma=0$),
\begin{equation}
\label{eq:force-information-derivative}
\frac{\dd}{\dd\gamma}\mathcal M_\gamma
=
\frac12\E\big\|
\widehat m_\gamma^R(Y_\gamma,R)-\widehat m_\gamma(Y_\gamma)
\big\|^2.
\end{equation}
Consequently,
\begin{equation}
\label{eq:force-information-integral}
\mathcal M_\gamma
=
\frac12\int_0^\gamma
\E\big\|
\widehat m_u^R(Y_u,R)-\widehat m_u(Y_u)
\big\|^2\,\dd u.
\end{equation}
At $\gamma_t=a_t^2/\Delta_t$,
\begin{equation}
\label{eq:force-energy-information-rate}
\E\|g_t(X_t,R)\|^2
=
\left.
\frac{2a_t^2}{\Delta_t^2}
\frac{\dd}{\dd\gamma}\mathcal M_\gamma
\right|_{\gamma=\gamma_t}.
\end{equation}
\end{proposition}

Because $\gamma_t=a_t^2/\Delta_t$ satisfies
$\dd\gamma_t/\dd t=-2a_t^2/\Delta_t^2$ under the VP schedule,
\eqref{eq:force-energy-information-rate} gives
\begin{equation}
\label{eq:physical-time-assimilation-rate}
-\frac{\dd}{\dd t}\mathcal M_{\gamma_t}
=\E\|g_t(X_t,R)\|^2,
\end{equation}
the measurement average of \eqref{eq:fixed-r-kl-dissipation}.  Hence
\(\mathcal M_{\gamma_t}\) decreases under forward noising, and the same
information is acquired under reverse denoising at the force-energy rate.

At zero channel SNR, the right derivative is
\begin{equation}
\label{eq:initial-measurement-information-slope}
\left.\frac{\dd}{\dd\gamma}\mathcal M_\gamma\right|_{\gamma=0^+}
=
\frac12\E\big\|\E[X_0\mid R]-\E X_0\big\|^2
=
\frac12\Tr\Cov\!\big(\E[X_0\mid R]\big).
\end{equation}
Thus the initial growth is determined by the clean-signal variance explained
by the measurement.  Combining
\eqref{eq:initial-measurement-information-slope} with
\eqref{eq:force-energy-information-rate} recovers the high-noise scaling in
Corollary~\ref{cor:general-high-noise-energy}, now with leading constant
\(\Tr\Cov(\E[X_0\mid R])\).  In the linear Gaussian model,
$\Cov(\E[X_0\mid R])=\Sigma A^\top Q^{-1}A\Sigma$, where
$Q=A\Sigma A^\top+\Gamma$, matching the leading coefficient in
\eqref{eq:likelihood-energy-leading}.

\(\mathcal M_\gamma\) averages over random measurements.  It is distinct from
the fixed-\(r\) empirical-index information introduced in
Section~\ref{sec:collapse}, which tracks the resolution of clean explanations
still compatible with one realized measurement.

\subsection{Linear-Gaussian Closed-Form Specialization}
\label{sec:gaussian-force-specialization}

The preceding identities do not require a Gaussian prior.  To make the role
of operator geometry explicit, we now specialize to the linear Gaussian
model, where the force and its energy are available in closed form.

\begin{proposition}[Linear Gaussian smoothed likelihood and force]
\label{thm:linear-gaussian-guidance}
Let \(X_0\sim\calN(0,\Sigma)\) and
\(R=AX_0+\varepsilon\), where \(\varepsilon\sim\calN(0,\Gamma)\),
\(\Gamma\succ0\), and \(\varepsilon\) is independent of \(X_0\) and of the
forward noising variable used to form \(X_t\).  Define
\begin{equation}
\label{eq:linear-gaussian-BQ}
B_t=a_t^2\Sigma+\Delta_t\Id_d,
\qquad
Q=A\Sigma A^\top+\Gamma.
\end{equation}
Then
\begin{equation}
\label{eq:linear-gaussian-smoothed-likelihood}
R\mid X_t=x\sim \calN(M_tx,\Sigma_{R\mid t}),
\end{equation}
where
\begin{equation}
\label{eq:linear-gaussian-M-Sigma}
M_t=a_tA\Sigma B_t^{-1},
\qquad
\Sigma_{R\mid t}=Q-a_t^2A\Sigma B_t^{-1}\Sigma A^\top.
\end{equation}
Consequently, the exact measurement force is
\begin{equation}
\label{eq:linear-gaussian-guidance-force}
g_t(x,r)=M_t^\top\Sigma_{R\mid t}^{-1}(r-M_tx).
\end{equation}
Under the joint law of \((X_t,R)\),
\begin{equation}
\label{eq:linear-gaussian-guidance-energy}
\frac1d\E\|g_t(X_t,R)\|^2
=
\frac1d\operatorname{Tr}
\left(M_t^\top\Sigma_{R\mid t}^{-1}M_t\right).
\end{equation}
Moreover, under the VP schedule \eqref{eq:forward-marginal} (so that
\(\Delta_t=1-a_t^2\)), if along a dimension sequence
\(\lambda_{\min}(\Gamma)\ge\gamma_0>0\) and
\(d^{-1}\|A\Sigma\|_F^2\le C_0\), then as \(a_t\to0\),
\begin{equation}
\label{eq:high-noise-guidance-decay}
\frac1d\E\|g_t(X_t,R)\|^2=O(a_t^2),
\end{equation}
which is the Gaussian specialization of
Corollary~\ref{cor:general-high-noise-energy}.
\end{proposition}

The formulas follow by conditioning the joint Gaussian pair \((X_t,R)\).  The
identity
\(\Sigma_{R\mid t}=A\Cov(X_0\mid X_t)A^\top+\Gamma\) shows that the residual in
\eqref{eq:linear-gaussian-guidance-force} is weighted by the inverse predictive
covariance, which accounts for both measurement noise and the uncertainty about
\(X_0\) remaining at noise level \(t\).  The factor \(M_t^\top\) then maps this
measurement-space quantity back to signal space.  In fixed dimension
\eqref{eq:linear-gaussian-guidance-force} gives \(g_t=O(a_t)\) for bounded
\(x,r\).  In high dimension the normalized energy is the appropriate
statement; Section~\ref{sec:consistency} further resolves it by row and
singular directions.

\begin{corollary}[Gaussian relative-energy decay]
\label{cor:relative-energy}
Under Proposition~\ref{thm:linear-gaussian-guidance}, the prior score is
\[
s_t^{\rm prior}(x)=-B_t^{-1}x.
\]
Then, under \(X_t\sim p_t\),
\begin{equation}
\label{eq:prior-score-energy}
\frac1d\E\|s_t^{\rm prior}(X_t)\|^2
=\frac1d\operatorname{Tr}\!\left(B_t^{-1}\right)
=\frac1d\sum_{i=1}^d\frac{1}{a_t^2\lambda_i+\Delta_t},
\end{equation}
where \(\{\lambda_i\}\) are the eigenvalues of \(\Sigma\).  Under the schedule
\eqref{eq:forward-marginal}, \(\Delta_t\to1\) as \(a_t\to0\), so this converges
to \(1\) and the prior-score energy is \(\Theta(1)\); the same conclusion holds
for any VP schedule with \(\Delta_t\to\Delta_\infty>0\).  In fixed
dimension, the likelihood-force energy has the leading-order expansion
\begin{equation}
\label{eq:likelihood-energy-leading}
\frac1d\E\|g_t(X_t,R)\|^2
=\frac{a_t^2}{\Delta_t^2}\,\frac1d\operatorname{Tr}\!\left(\Sigma A^\top Q^{-1}A\Sigma\right)
+O(a_t^4),
\end{equation}
and therefore
\begin{equation}
\label{eq:rho-ratio}
\rho_t^{\rm global}
:=
\frac{\E\|g_t(X_t,R)\|^2}
{\E\|s_t^{\rm prior}(X_t)\|^2}
=O(a_t^2).
\end{equation}
In the isotropic case \(\Sigma=\sigma_x^2\Id_d\), \(A=\Id_d\),
\(\Gamma=\sigma_y^2\Id_d\),
\(\rho_t^{\rm global}\sim a_t^2\sigma_x^4/(\sigma_x^2+\sigma_y^2)\).
\end{corollary}

The measurement term is present at every noise level, but at high noise its
energy is small relative to the prior score.  This statement concerns relative
vector-field strength; it does not impose a global monotonicity or a universal
scalar crossover time.

\runinhead{Scope of the energy statements.}
The high-noise bounds in Corollaries~\ref{cor:general-high-noise-energy}
and~\ref{cor:general-relative-energy} average over the joint law of
$(X_t,R)$.  At fixed \(r\), the energy
\(\E_{X_t\mid r}\|g_t(X_t,r)\|^2\) has the exact interpretation
\eqref{eq:fixed-r-kl-dissipation}, while its high-noise coefficient may depend
on \(r\).  The Fisher decomposition in
Lemma~\ref{prop:joint-fisher-pythagoras} is likewise a joint-law
orthogonality and need not hold at fixed \(r\).

When the measurement force is concentrated in a low-dimensional row or
spectral subspace, a full-space normalization can be diluted by prior-score
energy in the remaining directions.  Section~\ref{sec:consistency} therefore
introduces row and direction-wise ratios alongside the global summary.

%% file: sections/solvable_reductions.tex
We now use three solvable models to isolate how conditioning modifies the
unconditional regime picture.  A conditional Gaussian mixture separates class
information supplied by the measurement from class information acquired during
denoising.  A finite empirical posterior gives an exact accounting of the
sample-index uncertainty that remains after measurement.  A linear Gaussian
model resolves the measurement force across operator--prior directions.  Each
model isolates one mechanism rather than approximating a common data
distribution.  Proofs are given in Appendix~A.

\subsection{Residual Class Information in Conditional Gaussian Mixtures}
\label{sec:speciation}
\input{sections/speciation}

\subsection{Empirical Explanation Resolution}
\label{sec:collapse}
\input{sections/collapse}

\subsection{Directional Information Allocation and Operator--Prior Alignment}
\label{sec:consistency}
\input{sections/consistency}

%% file: sections/speciation.tex
The first model asks how conditioning changes class-level speciation.  A
measurement may already favor one class, so raw class commitment no longer
isolates information acquired during denoising.  In a symmetric Gaussian
mixture, the distinction is exact: conditioning induces a static field in the
class weights and replaces the original class separation by the part left
unresolved by the measurement.  Only the latter can generate new class
commitment along the reverse trajectory.  Let
\begin{equation}
\label{eq:binary-class-prior}
C\in\{+1,-1\},\qquad P(C=+1)=P(C=-1)=1/2,
\end{equation}
and
\begin{equation}
\label{eq:binary-gmm-prior}
X_0\mid C=c\sim\calN(cm,\Sigma).
\end{equation}
The measurement is \(R=AX_0+\varepsilon\), with
\(\varepsilon\sim\calN(0,\Gamma)\) and \(\Gamma\succ0\).  Define
\begin{equation}
\label{eq:gmm-QK}
Q=A\Sigma A^\top+\Gamma,
\qquad
K=\Sigma A^\top Q^{-1}.
\end{equation}
Here \(Q\) is the within-class measurement covariance; the marginal covariance
of \(R\) also contains the between-class contribution \(Amm^\top A^\top\).

\begin{proposition}[Posterior Gaussian-mixture geometry]
\label{thm:conditional-gmm}
For \(c\in\{\pm1\}\),
\begin{equation}
\label{eq:gmm-clean-posterior}
X_0\mid r,C=c\sim\calN(Kr+c\delta,S),
\end{equation}
where
\begin{equation}
\label{eq:gmm-posterior-S-delta}
S=\Sigma-\Sigma A^\top Q^{-1}A\Sigma,
\qquad
\delta=(\Id_d-KA)m.
\end{equation}
The posterior class weights are
\begin{equation}
\label{eq:posterior-class-weights}
P(C=\pm1\mid r)
=
\alpha_\pm(r)
=
\frac{e^{\pm \eta(r)}}{2\cosh \eta(r)},
\qquad
\eta(r)=(Am)^\top Q^{-1}r.
\end{equation}
With \(\Lambda_t=a_t^2S+\Delta_t\Id_d\) and \(Z_t=X_t-a_tKr\), write
\(\widetilde p_t(z\mid r)=p_t(z+a_tKr\mid r)\) for the noised posterior density
in the shifted coordinate.  Then
\begin{equation}
\label{eq:gmm-noised-posterior-mixture}
\widetilde p_t(z\mid r)
=
\alpha_+\calN(z;a_t\delta,\Lambda_t)
+
\alpha_-\calN(z;-a_t\delta,\Lambda_t),
\end{equation}
or equivalently
\begin{equation}
\label{eq:gmm-external-field-form}
\widetilde p_t(z\mid r)
\propto
\exp\{-z^\top \Lambda_t^{-1}z/2\}
\cosh\{\eta(r)+a_t\delta^\top \Lambda_t^{-1}z\}.
\end{equation}
\end{proposition}

The shift \(Kr\) moves the posterior center.  The scalar \(\eta(r)\) is the
class evidence already contained in the measurement, while
\(\delta=(\Id_d-KA)m\) is the class separation that remains available to the
noisy state.  Thus \(\eta(r)\) acts as a static external field and \(\delta\)
governs dynamic speciation.  When \(A=0\), \(K=0\), \(\eta=0\),
\(\delta=m\), and \(S=\Sigma\), recovering the unconditional symmetric-mixture
analysis of \citet{biroli2024dynamical}.

\begin{corollary}[Posterior score and dynamic class evidence]
\label{cor:conditional-gmm-score-evidence}
Under the model of Proposition~\ref{thm:conditional-gmm}, the posterior score in
the shifted coordinate is
\begin{equation}
\label{eq:conditional-gmm-score}
\nabla_z\log \widetilde p_t(z\mid r)
=
-\Lambda_t^{-1}z
+
a_t\Lambda_t^{-1}\delta
\tanh\{\eta(r)+a_t\delta^\top \Lambda_t^{-1}z\},
\end{equation}
and the posterior class evidence decomposes as
\begin{equation}
\label{eq:posterior-class-evidence}
\frac12
\log\frac{P(C=+1\mid z,r)}{P(C=-1\mid z,r)}
=
\eta(r)+a_t\delta^\top \Lambda_t^{-1}z.
\end{equation}
\end{corollary}

The log-odds thus split into a static part and a part created by the current
noisy state: raw class commitment mixes measurement-supplied information with
information revealed dynamically by denoising.

\begin{remark}[Shared-covariance multiclass extension]
The same re-referencing holds beyond two classes.  If
$C\in\{1,\ldots,K_c\}$, $P(C=c)=\pi_c$, and
$X_0\mid C=c\sim\calN(m_c,\Sigma)$ with a shared covariance, then the same
$K$ and $S$ give
\[
X_0\mid r,C=c\sim
\calN\!\left(Kr+(\Id_d-KA)m_c,S\right),
\]
while
\[
P(C=c\mid r)\propto
\pi_c\exp\!\left\{-\tfrac12(r-Am_c)^\top Q^{-1}(r-Am_c)\right\}.
\]
Every pairwise class separation is therefore replaced by
$(\Id_d-KA)(m_c-m_{c'})$, and the noised posterior score averages the residual
means with softmax weights.  The binary symmetric case is special because it
reduces to one coordinate and admits a pitchfork reference.  With
class-dependent covariances, the gains and residual covariances also depend on
the class, and neither the common residual map nor a single
\(\kappa_t\)-threshold remains.
\end{remark}

\begin{proposition}[Baseline-corrected posterior cloning]
\label{prop:posterior-speciation}
Define the conditional cloning observable
\begin{equation}
\label{eq:conditional-cloning}
\phi_r(t)
=
\E_{Z_t\mid r}
\left[
\sum_{c\in\{\pm1\}}P(C=c\mid Z_t,r)^2
\right].
\end{equation}
Then
\begin{equation}
\label{eq:cloning-baseline}
\phi_r(\infty)
=
\alpha_+^2+
\alpha_-^2
=
\sum_cP(C=c\mid r)^2.
\end{equation}
Therefore the dynamic class-speciation signal is
\begin{equation}
\label{eq:baseline-corrected-cloning}
\Delta\phi_r(t)=\phi_r(t)-\phi_r(\infty),
\end{equation}
or, when \(\phi_r(\infty)<1\),
\begin{equation}
\label{eq:normalized-excess-cloning}
\widetilde\phi_r(t)
=
\frac{\phi_r(t)-\phi_r(\infty)}
{1-\phi_r(\infty)}.
\end{equation}
\end{proposition}

By the Markov property of the exact reverse dynamics, \(\phi_r(t)\) is the
probability that two independent exact-reverse continuations launched from
the same state commit to the same class---the conditional analogue of the
cloning experiment of \citet{biroli2024dynamical}, with baseline
\(\alpha_+^2+\alpha_-^2\) rather than the uninformed \(1/2\).

\begin{corollary}[One-dimensional reduction and zero-field reference]
\label{cor:posterior-speciation-coordinate}
Under the model of Proposition~\ref{thm:conditional-gmm}, with \(\phi_r\) as in
\eqref{eq:conditional-cloning}, let
\begin{equation}
\label{eq:posterior-speciation-strength}
\kappa_t=a_t^2\delta^\top \Lambda_t^{-1}\delta,
\qquad
G_\pm(s)=\calN(s;\pm\sqrt{\kappa_t},1).
\end{equation}
Then
\begin{equation}
\label{eq:one-dimensional-cloning-integral}
\phi_r(t)
=
\int_{\R}
\frac{\alpha_+^2G_+(s)^2+\alpha_-^2G_-(s)^2}
{\alpha_+G_+(s)+\alpha_-G_-(s)}\,\dd s.
\end{equation}
In the zero-field reference case \(\eta(r)=0\), the origin \(z=0\) is a local
maximum of \(\log \widetilde p_t(z\mid r)\) if and only if
\(\kappa_t\le1\); for
\(\kappa_t>1\) it loses local maximality, becoming a local minimum in the
one-dimensional reduction and a saddle when \(d>1\).  Since
\(\kappa_t\) is nondecreasing along the reverse trajectory, and strictly
increasing whenever \(\delta\neq0\), the double-well landscape forms as
\(\kappa_t\) first crosses the threshold
\begin{equation}
\label{eq:posterior-speciation-threshold}
\kappa_t=a_t^2\delta^\top \Lambda_t^{-1}\delta=1.
\end{equation}
\end{corollary}

The level \(\kappa_t=1\) describes a local change in the posterior
log-density, and hence in the exact reverse drift, rather than a transition of
every stochastic path.  It is specific to the symmetric zero-field reference.
When \(\eta(r)\ne0\) and the residual separation is nonzero, the origin is no
longer stationary and the pitchfork is replaced by a biased crossover;
\(\Delta\phi_r\) or
\(\widetilde\phi_r\) then measures class information acquired beyond the
measurement baseline.  The model proves monotonicity of \(\kappa_t\), but makes
no corresponding claim for the baseline-corrected cloning observable at
nonzero field.  When \(A=0\), the criterion reduces to the unconditional one
with \(\delta=m\) and \(S=\Sigma\).

The residual nature of speciation is especially transparent in singular
coordinates.  If \(\Sigma=\sigma_x^2\Id_d\), \(\Gamma=\sigma_y^2\Id_{d_y}\), and
\(A=U\diag(s_i)V^\top\), write \(m_i=v_i^\top m\) and
\(\delta_i=v_i^\top\delta\).  Then
\begin{equation}
\label{eq:residual-coordinate-shrinkage}
\delta_i=\frac{\sigma_y^2}{\sigma_y^2+\sigma_x^2s_i^2}m_i.
\end{equation}
When \(s_i^2\sigma_x^2\gg\sigma_y^2\), the residual \(\delta_i\) is strongly
attenuated.  Null directions satisfy \(s_i=0\) and retain
\(\delta_i=m_i\).  Conditioning therefore reallocates class-level speciation
toward directions that the measurement has not resolved.

%% file: sections/collapse.tex
The second model concerns resolution of individual explanations in a finite
empirical support.  Here ``collapse'' retains the meaning used in the
unconditional regime theory: the noisy state progressively identifies a
support index.  Conditioning changes the candidate set and its entropy.  Let
\begin{equation}
\label{eq:collapse-empirical-prior}
P_0^{\rm emp}=\frac1n\sum_{i=1}^n\delta_{x_i},
\end{equation}
and let \(I\) denote the support index.  Under the linear Gaussian likelihood
of Section~\ref{sec:background}, the posterior weights at a fixed measurement
\(r\) are
\begin{equation}
\label{eq:empirical-posterior-weights}
w_i(r)=\Prob(I=i\mid r)
=
\frac{
\exp\{-\frac12\|Ax_i-r\|_{\Gamma^{-1}}^2\}
}{
\sum_j \exp\{-\frac12\|Ax_j-r\|_{\Gamma^{-1}}^2\}
}.
\end{equation}
For \(\Gamma=\sigma_y^2\Id_{d_y}\), this reduces to
\begin{equation}
\label{eq:isotropic-empirical-weights}
w_i(r)\propto
\exp\left\{-\frac{\|Ax_i-r\|^2}{2\sigma_y^2}\right\}.
\end{equation}
For \(t>0\), noising the posterior empirical distribution gives
\begin{equation}
\label{eq:noised-empirical-posterior}
p_t^{\rm emp}(x\mid r)
=
\sum_{i=1}^n w_i(r)\calN(x;a_tx_i,\Delta_t\Id_d).
\end{equation}
Nonuniform empirical priors are handled by multiplying the likelihood factors
in \eqref{eq:empirical-posterior-weights} by the prior index probabilities.

\begin{proposition}[Posterior index-information accounting]
\label{thm:posterior-collapse}
For \(t>0\), the posterior responsibility of index \(i\) at a noisy state
\(X_t=x\) is
\begin{equation}
\label{eq:posterior-index-responsibility}
\Prob(I=i\mid x,r)
=
\frac{
w_i(r)\exp\{-\|x-a_tx_i\|^2/(2\Delta_t)\}
}{
\sum_j w_j(r)\exp\{-\|x-a_tx_j\|^2/(2\Delta_t)\}
}.
\end{equation}
Moreover,
\begin{equation}
\label{eq:posterior-collapse-entropy-identity}
h(X_t\mid r)
=
\frac d2\log(2\pi e\Delta_t)
+
H(I\mid r)
-
H(I\mid X_t,r).
\end{equation}
\end{proposition}

The identity is stated for \(t>0\) because the empirical posterior is discrete
at \(t=0\).  If the support points are pairwise distinct, then
\(H(I\mid X_t,r)\to0\) as \(t\to0\); duplicated points leave index ambiguity
that the reverse dynamics cannot resolve.  Averaging the pointwise entropy
\(H(I\mid X_t=x,r)\) over draws from the noised posterior gives an unbiased
Monte Carlo estimate of \(H(I\mid X_t,r)\).  This notion of collapse is
sample-index decoding among measurement-compatible points; it does not define
a continuous-space MAP estimator or neural mode collapse.

Without conditioning, a uniform empirical prior has index entropy \(\log n\).
After observing \(r\), the residual budget is
\[
H(I\mid r)=-\sum_i w_i(r)\log w_i(r)\le \log n.
\]
For a fixed measurement,
\begin{equation}
\label{eq:pointwise-index-information-gain}
\log n-H(I\mid r)
=
\KL\big(P(I\mid r)\,\|\,\mathrm{Unif}_n\big).
\end{equation}
Corollary~\ref{cor:integrated-head-start} gives this static head start an exact
dynamical representation:
\begin{equation}
\label{eq:collapse-head-start-force-budget}
\log n-H(I\mid r)
=
\int_0^\infty
\E_{X_t\mid r}\|g_t(X_t,r)\|^2\,\dd t.
\end{equation}
Thus the measurement-induced reduction of the empirical index budget is also
the total fixed-$r$ posterior--prior force-energy budget along the noising
path.  Averaging over the random measurement gives
\begin{equation}
\label{eq:average-index-information-gain}
H(I\mid R)=\log n-\MI(I;R).
\end{equation}
Adding the information contributed by the current noisy state yields
\begin{equation}
\label{eq:two-stage-index-information}
H(I\mid R,X_t)
=
\log n-\MI(I;R)-\MI(I;X_t\mid R).
\end{equation}
Hence the budget is determined by the entropy of the compatible explanation
set: informative measurements leave few plausible indices, weak measurements
leave a budget near \(\log n\), and at \(A=0\) the unconditional budget of
\citet{biroli2024dynamical} is recovered.

Define the cumulative posterior-resolution information at a fixed measurement,
\begin{equation}
\label{eq:index-information-acquired}
\mathcal J_t(r)
:=
\MI(I;X_t\mid R=r)
=
H(I\mid r)-H(I\mid X_t,r).
\end{equation}
This finite-\(n\) chain-rule quantity is cumulative rather than instantaneous.
The exact decomposition
\[
\log n-H(I\mid X_t,r)
=
\big[\log n-H(I\mid r)\big]+\mathcal J_t(r)
\]
separates the pointwise measurement head start from the explanation resolution
supplied by the noisy state.  For distinct support points,
\(\mathcal J_t(r)\to H(I\mid r)\) at the clean endpoint.  This process is
distinct from the measurement-assimilation information
\(\mathcal M_\gamma=\MI(R;Y_\gamma)\) of
Proposition~\ref{prop:force-information-rate}; their exact relation after
averaging over \(R\) is given in Section~\ref{sec:regime-map-synthesis}.

We next use a Gaussian conditional-information surrogate to obtain a
descriptive time scale for this finite-support process.

\begin{remark}[Gaussian matching scale for conditional index resolution]
\label{cor:collapse-time}
Suppose that, before empirical components separate, the clean posterior is
approximated by
\begin{equation}
\label{eq:continuous-posterior-approx}
X_0\mid r\approx\calN(\mu_r,S_r).
\end{equation}
Define the Gaussian conditional-information surrogate
\begin{equation}
\label{eq:gaussian-conditional-information}
\mathcal I_t^{\rm G}(r)
:=
\MI_{\rm G}(X_0;X_t\mid R=r)
=
\frac12\log\det
\left(\Id_d+\frac{a_t^2}{\Delta_t}S_r\right).
\end{equation}
For \(q\in(0,1]\), define the fractional Gaussian resolution time
\(t_q^{\rm G}(r)\), when it exists, by
\begin{equation}
\label{eq:fractional-gaussian-resolution-time}
\mathcal I_{t_q^{\rm G}}^{\rm G}(r)=qH(I\mid r).
\end{equation}
If \(S_r=s_r\Id_d\), \(s_r>0\), and
\(\alpha_r=d^{-1}H(I\mid r)>0\), then
\begin{equation}
\label{eq:isotropic-fractional-resolution-time}
t_q^{\rm G}(r)
=
\frac12\log
\left[
1+\frac{s_r}{e^{2q\alpha_r}-1}
\right].
\end{equation}
The full-budget matching scale is \(t_C(r):=t_1^{\rm G}(r)\).
\end{remark}

The estimate does not follow from Proposition~\ref{thm:posterior-collapse}: it
replaces a finite weighted codebook by a Gaussian population channel and
identifies resolution scales rather than a universal sharp transition.
The likelihood-width study in Appendix~\ref{sec:app-exact-details}
evaluates the fractional case
\(q=1/2\); the full-budget case \(q=1\) is more sensitive to the
small-eigenvalue tail of \(S_r\).  If \(H(I\mid r)=0\), the measurement has
already identified the index.  More generally, conditioning changes both the
budget and \(S_r\), so budget reduction alone does not determine how the
matching scale shifts.

%% file: sections/consistency.tex
The third model asks how measurement influence is distributed across the
reverse vector field.  The preceding class and index quantities describe what
uncertainty remains; force-to-prior ratios describe the directions in which
the measurement-induced field becomes appreciable.  For a low-rank operator,
unobserved directions can dominate a global prior-score norm, making row-space
and direction-wise ratios more informative than a single aggregate.

\runinhead{Global and fixed-measurement ratios.}
The joint-law aggregate ratio is
\begin{equation}
\label{eq:global-rho-consistency}
\rho_t^{\rm global}
=
\frac{\E\|g_t(X_t,R)\|^2}
{\E\|s_t^{\rm prior}(X_t)\|^2},
\qquad
s_t^{\rm prior}(x)=\nabla_x\log p_t(x).
\end{equation}
Corollary~\ref{cor:general-relative-energy} gives
\(\rho_t^{\rm global}=O(a_t^2)\) at high noise under the same
per-coordinate second-moment condition as the force-energy bound.  For a realized measurement we use, whenever the
denominator is positive,
\begin{equation}
\label{eq:fixed-r-rho}
\rho_t^{\rm global}(r)
=
\frac{\E_{X_t\mid r}\|g_t(X_t,r)\|^2}
{\E_{X_t\mid r}\|s_t^{\rm prior}(X_t)\|^2},
\end{equation}
a trajectory diagnostic rather than a pointwise consequence of the joint-law
bound.

\runinhead{Row, null, and singular directions.}
With \(P_A=A^\dagger A\) and \(P_A^\perp=\Id_d-P_A\), define, whenever the
corresponding denominator is positive,
\begin{equation}
\label{eq:obs-rho}
\rho_t^{\rm row}(r)
=
\frac{\E_{X_t\mid r}\|P_Ag_t\|^2}
{\E_{X_t\mid r}\|P_As_t^{\rm prior}\|^2},
\qquad
\rho_t^{\rm null}(r)
=
\frac{\E_{X_t\mid r}\|P_A^\perp g_t\|^2}
{\E_{X_t\mid r}\|P_A^\perp s_t^{\rm prior}\|^2}.
\end{equation}
For anisotropic operators the direction-wise form is more informative.  Let
\(v_i\) be a right singular vector of \(A\), with singular value \(s_i\), and
write
\[
g_{t,i}(x,r)=v_i^\top g_t(x,r),
\qquad
s_{t,i}^{\rm prior}(x)=v_i^\top s_t^{\rm prior}(x).
\]
Then
\begin{equation}
\label{eq:spectral-rho}
\rho_i(t,r)
=
\frac{\E_{X_t\mid r}|g_{t,i}(X_t,r)|^2}
{\E_{X_t\mid r}|s_{t,i}^{\rm prior}(X_t)|^2},
\qquad
\rho_t^{\rm row}(r)
=
\frac{\sum_{i:s_i>0}\E_{X_t\mid r}|g_{t,i}|^2}
{\sum_{i:s_i>0}\E_{X_t\mid r}|s_{t,i}^{\rm prior}|^2},
\end{equation}
again when the denominators are positive.  These are continuous measures of
relative vector-field strength, not intrinsic phase boundaries.  When the
expectation is instead over the joint law of \((X_t,R)\), we write
\begin{equation}
\label{eq:joint-spectral-rho}
\rho_i(t)
=
\frac{\E|g_{t,i}(X_t,R)|^2}
{\E|s_{t,i}^{\rm prior}(X_t)|^2},
\qquad
\rho_t^{\rm row}
=
\frac{\sum_{i:s_i>0}\E|g_{t,i}(X_t,R)|^2}
{\sum_{i:s_i>0}\E|s_{t,i}^{\rm prior}(X_t)|^2}.
\end{equation}

\begin{theorem}[Aligned anisotropic Gaussian allocation]
\label{prop:spectral-consistency}
Let $X_0\sim\calN(0,\Sigma)$, where
\[
\Sigma=V\diag(\lambda_1,\ldots,\lambda_d)V^\top,
\qquad \lambda_i>0.
\]
Let $A\in\R^{d_y\times d}$ have full row rank and suppose the measurement
is aligned with the prior eigenbasis,
\[
A=U\begin{bmatrix}\diag(s_1,\ldots,s_{d_y})&0\end{bmatrix}V^\top,
\qquad
\Gamma=\sigma_y^2\Id_{d_y},
\]
with $s_i=0$ for $i>d_y$.  Define
\begin{equation}
\label{eq:anisotropic-Bv}
B_{i,t}=a_t^2\lambda_i+\Delta_t,
\qquad
v_{i,t}=\frac{\lambda_i\Delta_t}{B_{i,t}},
\qquad
\tau_i^2=\left(\lambda_i^{-1}+s_i^2/\sigma_y^2\right)^{-1},
\qquad
\widetilde B_{i,t}=a_t^2\tau_i^2+\Delta_t,
\end{equation}
where $\tau_i^2=\lambda_i$ when $s_i=0$.  Under the joint law of $(X_t,R)$,
\begin{equation}
\label{eq:anisotropic-spectral-energies}
\E|s_{t,i}^{\rm prior}(X_t)|^2=\frac1{B_{i,t}},
\qquad
\E|g_{t,i}(X_t,R)|^2
=
\frac{a_t^2s_i^2\lambda_i^2}
{B_{i,t}^2(s_i^2v_{i,t}+\sigma_y^2)},
\end{equation}
and
\begin{equation}
\label{eq:anisotropic-rho}
\rho_i(t)
=
\frac{a_t^2s_i^2\lambda_i^2}
{B_{i,t}(s_i^2v_{i,t}+\sigma_y^2)}
=
\frac{B_{i,t}}{\widetilde B_{i,t}}-1.
\end{equation}
Consequently:
\begin{enumerate}
\item As $a_t\to0$,
\[
\rho_i(t)=
\frac{a_t^2s_i^2\lambda_i^2}{s_i^2\lambda_i+\sigma_y^2}+O(a_t^4).
\]
\item For $s_i>0$, $\rho_i(t)$ is strictly increasing along the reverse
trajectory, with low-noise limit $s_i^2\lambda_i/\sigma_y^2$, and
\begin{equation}
\label{eq:variance-halving-reference}
\rho_i(t)=1\quad\Longleftrightarrow\quad
\widetilde B_{i,t}=\frac12B_{i,t}.
\end{equation}
Thus, within this aligned linear-Gaussian model, $\rho_i=1$ is a
variance-halving, equal-Fisher-energy reference level.  It is crossed at a
unique interior noise level iff $s_i^2\lambda_i>\sigma_y^2$; when
$s_i^2\lambda_i=\sigma_y^2$, it is attained only at the clean endpoint.
\item Since the exact force is null on unmeasured aligned coordinates,
\begin{equation}
\label{eq:anisotropic-weighted-dilution}
\rho_t^{\rm global}=\omega_t\rho_t^{\rm row},
\qquad
\omega_t=
\frac{\sum_{i:s_i>0}B_{i,t}^{-1}}{\sum_{i=1}^dB_{i,t}^{-1}}.
\end{equation}
For an isotropic prior, $\omega_t=d_y/d$.
\end{enumerate}
\end{theorem}

The identity \eqref{eq:anisotropic-rho} also follows from the directional
Fisher-energy decomposition: the posterior noisy variance is
$\widetilde B_{i,t}$, so the posterior and prior Fisher energies are
$\widetilde B_{i,t}^{-1}$ and $B_{i,t}^{-1}$.  The equal-energy level therefore
has a posterior-contraction meaning; it is not an intrinsic phase boundary.

At the clean endpoint, maximizing Gaussian measurement information over a
row-orthonormal design is the classical D-optimal principal-subspace problem
\citep{chaloner1995bayesian}.  The dynamical question is whether this
orientation ordering persists throughout the noising family, where the
represented measurement information is only partial.

\begin{theorem}[Matched-spectrum orientation dependence]
\label{thm:matched-spectrum-alignment}
Let $X_0\sim\calN(0,\Sigma)$ with eigenvalues
$\lambda_1\ge\cdots\ge\lambda_d>0$, let $1\le d_y\le d$, let
$A\in\R^{d_y\times d}$, and let $R_A=AX_0+\varepsilon$ with
$\varepsilon\sim\calN(0,\sigma_y^2\Id_{d_y})$.  Assume
$AA^\top=\Id_{d_y}$.  Thus every admissible $A$ has the same nonzero singular
values; isotropic measurement noise leaves orientation relative to \(\Sigma\)
as the only design variable.  Define
\[
C_\gamma=(\Sigma^{-1}+\gamma\Id_d)^{-1},
\qquad
\Sigma_A^{\rm post}
=
(\Sigma^{-1}+\sigma_y^{-2}A^\top A)^{-1},
\qquad
\psi_\gamma(\lambda)
=
\frac12\log
\frac{\sigma_y^2+\lambda}
{\sigma_y^2+\lambda/(1+\gamma\lambda)}.
\]
Then
\begin{equation}
\label{eq:orientation-assimilation-general}
\mathcal M_\gamma(A)=\MI(R_A;Y_\gamma)
=
\frac12\log
\frac{\det(\sigma_y^2\Id_{d_y}+A\Sigma A^\top)}
{\det(\sigma_y^2\Id_{d_y}+AC_\gamma A^\top)}.
\end{equation}
Moreover, among all row-orthonormal rank-$d_y$ measurements,
\begin{align}
\label{eq:orientation-clean-information-bounds}
\frac12\sum_{i=d-d_y+1}^{d}
\log\!\left(1+\frac{\lambda_i}{\sigma_y^2}\right)
&\le \MI(X_0;R_A)
\le
\frac12\sum_{i=1}^{d_y}
\log\!\left(1+\frac{\lambda_i}{\sigma_y^2}\right),\\
\label{eq:orientation-volume-contraction}
2\MI(X_0;R_A)
&=
\log\frac{\det\Sigma}{\det\Sigma_A^{\rm post}},\\
\label{eq:orientation-assimilation-bounds}
\sum_{i=d-d_y+1}^{d}\psi_\gamma(\lambda_i)
&\le \mathcal M_\gamma(A)
\le
\sum_{i=1}^{d_y}\psi_\gamma(\lambda_i).
\end{align}
For every $\gamma>0$, $\psi_\gamma$ is strictly increasing.  Equality in the
upper (respectively, lower) bounds is attained by choosing the row space of
$A$ to be the span of the largest (respectively, smallest) $d_y$ prior
eigenvectors.  If the boundary eigenvalue is repeated, the extremizing row
space need not be unique: the required dimensions may be chosen arbitrarily
within that eigenspace.  Hence equal operator
singular values do not determine clean
measurement information, posterior covariance-volume contraction, or
reverse-path assimilation once the prior is anisotropic.
\end{theorem}

\begin{corollary}[Rank-one orientation contrast]
\label{cor:rank-one-orientation-contrast}
Let $d_y=1$, let $A_i=v_i^\top$ measure the $i$th prior eigenvector, and let
$\lambda_i>\lambda_j$.  Then $A_i$ and $A_j$ have the same singular value,
but for every $\gamma>0$,
\begin{equation}
\label{eq:rank-one-orientation-contrast}
\mathcal M_\gamma(A_i)-\mathcal M_\gamma(A_j)
=
\psi_\gamma(\lambda_i)-\psi_\gamma(\lambda_j)>0.
\end{equation}
They also satisfy
$\MI(X_0;R_{A_i})>\MI(X_0;R_{A_j})$ and
$\det\Sigma_{A_i}^{\rm post}<\det\Sigma_{A_j}^{\rm post}$.
\end{corollary}

By Corollary~\ref{cor:reverse-path-kl},
$\mathcal M_{\gamma_t}(A)$ is also the mean posterior--prior reverse-path
relative entropy down to time $t$.  The theorem therefore gives an exact
Gaussian operator--prior alignment effect at the path level: spectrum-only
sufficiency is a symmetry consequence of the isotropic prior, not a property
of the measurement spectrum alone.  Under an equal row-space budget, choosing
the leading prior eigenspace maximizes assimilation at every noise level, not
only at the clean endpoint.  Figure~\ref{fig:anisotropic-alignment} illustrates
the rank-one case.

\begin{figure}[t]
\centering
\includegraphics[width=0.84\linewidth]{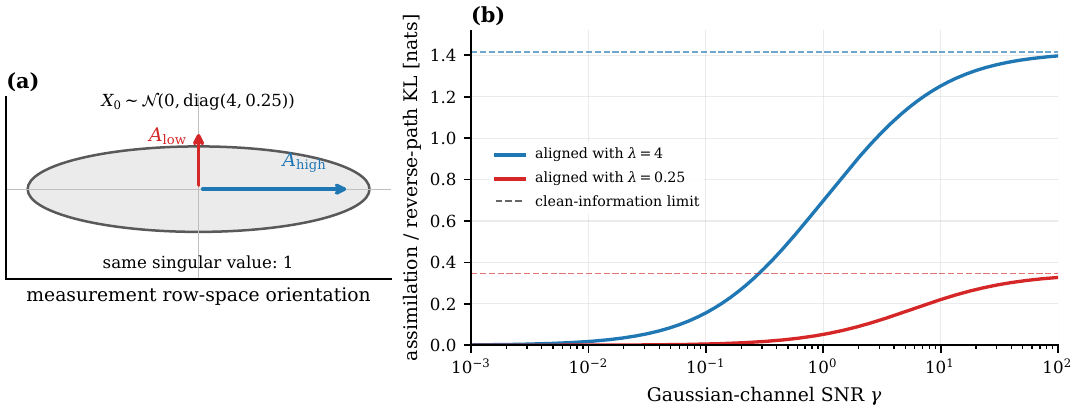}
\caption{Matched-spectrum orientation dependence in an anisotropic Gaussian
prior.  (a)~Two rank-one, row-orthonormal measurements have the same singular
values but select different prior eigendirections.  (b)~Their exact
measurement-assimilation curves
$\mathcal M_{\gamma}(A)=\MI(R_A;Y_\gamma)$ differ throughout the Gaussian
channel and equal the corresponding mean posterior--prior reverse-path KL at
$\gamma=\gamma_t$; the dashed asymptotes are the clean measurement
information, equivalently half the log posterior-covariance volume contraction.
Parameters are
$\Sigma=\operatorname{diag}(4,0.25)$ and $\sigma_y=0.5$; the curves evaluate
Theorem~\ref{thm:matched-spectrum-alignment} in closed form.}
\label{fig:anisotropic-alignment}
\end{figure}

\begin{proposition}[Rank dilution under coordinate-homogeneous score energy]
\label{prop:null-space-dilution}
Fix \(r\), let \(r_A=\operatorname{rank}(A)\), and use an orthonormal
coordinate system aligned with the row space of \(A\).  Assume (i) \(P_A^\perp g_t(\cdot,r)=0\) and (ii)
\(\E_{X_t\mid r}|s_{t,i}^{\rm prior}(X_t)|^2=e_t(r)>0\) for every coordinate.
Then
\begin{equation}
\label{eq:global-row-dilution}
\rho_t^{\rm global}(r)=\frac{r_A}{d}\rho_t^{\rm row}(r).
\end{equation}
\end{proposition}

For \(4\times\) super-resolution, \(r_A/d=1/16\): under the proposition's
hypotheses a global ratio understates the row-space ratio by a factor of
sixteen.  This is an exact isotropic calibration, not a universal law.  In the
aligned anisotropic Gaussian model the exact factor is the Fisher-energy weight
\(\omega_t\) in \eqref{eq:anisotropic-weighted-dilution}.  More generally,
prior coupling across the Euclidean row and null spaces---including a
nonaligned anisotropic Gaussian prior---can give the exact force a nonzero null
component.  These limitations motivate using row/null projectors as diagnostic
coordinates without assuming that conditioning is confined to the row space.

\begin{lemma}[Projected total-covariance decomposition]
\label{prop:posterior-mean-dispersion-calibration}
Fix a measurement value $r$ and $t>0$, assume
$\E[\|X_0\|^2\mid R=r]<\infty$, and let $P$ be an orthogonal projector.  With
$m_t^r(X_t)=\E[X_0\mid X_t,R=r]$,
\begin{equation}
\label{eq:projected-total-covariance}
\begin{aligned}
\Tr\!\left(P\Cov(X_0\mid r)P\right)
&=
\E\!\left[
\Tr\!\left(P\Cov(X_0\mid X_t,r)P\right)
\,\middle|\,R=r
\right]\\
&\quad+
\Tr\Cov\!\left(Pm_t^r(X_t)\mid R=r\right).
\end{aligned}
\end{equation}
The first term is conditional variance that remains after observing the noisy
state; the second is variation of the exact posterior mean across noisy states.
Under the Gaussian noising family, the second term tends to zero as
the channel SNR tends to zero and to
$\Tr(P\Cov(X_0\mid r)P)$ as the channel SNR tends to infinity.  Thus
across-state clean-estimate dispersion is the \emph{resolved} component of a
total-covariance decomposition, not the remaining posterior covariance.
\end{lemma}

%% file: sections/regimes.tex
Sections~4 and~5 introduced several exact quantities: fixed-measurement
posterior--prior divergence, measurement-averaged information, residual class
and sample-index uncertainty, and direction-wise force ratios.  We now connect
these quantities before turning to learned priors.  For an empirical prior, a
chain-rule identity partitions the information represented in the noisy state.
For approximate samplers, a Gaussian moment calculation identifies the
predictive quantities that their measurement updates approximate.

\subsection{Measurement Assimilation and Index Resolution}
\label{sec:regime-map-synthesis}

Consider the uniform empirical prior on distinct points, with \(X_0=x_I\) and
\(I\) uniform on \(\{1,\ldots,n\}\).  The measurement \(R\) and noisy channel
state \(Y_\gamma\) contain information about the same support index.  The
measurement-assimilation information
\(\mathcal M_\gamma=\MI(R;Y_\gamma)\) records the part shared with the
measurement.  Define the complementary resolution information
\begin{equation}
\label{eq:measurement-averaged-resolution-information}
\overline{\mathcal J}_\gamma:=\MI(I;Y_\gamma\mid R),
\end{equation}
which records what the noisy state reveals about the index beyond the
measurement.

Because \(R\perp Y_\gamma\mid I\), the mutual-information chain rule gives the
exact partition
\begin{equation}
\label{eq:assimilation-resolution-decomposition}
\MI(I;Y_\gamma)
=\MI(I,R;Y_\gamma)
=\mathcal M_\gamma+\overline{\mathcal J}_\gamma.
\end{equation}
At the diffusion-channel SNR \(\gamma_t=a_t^2/\Delta_t\), the invertible
rescaling \(Y_{\gamma_t}=X_t/\sqrt{\Delta_t}\) also gives
\begin{equation}
\label{eq:channel-time-index-resolution}
\overline{\mathcal J}_{\gamma_t}
=\E_R\mathcal J_t(R),
\qquad
\mathcal J_t(r)=\MI(I;X_t\mid R=r).
\end{equation}
Thus \(\mathcal M_{\gamma_t}\) and \(\E_R\mathcal J_t(R)\) are two parts of
the same index information represented at noise level \(t\).  At
differentiability points, their rates partition in the same way:
\begin{equation}
\label{eq:assimilation-resolution-rate-decomposition}
\frac{\dd}{\dd\gamma}\MI(I;Y_\gamma)
=\mathcal M_\gamma'+\overline{\mathcal J}_\gamma'.
\end{equation}

All three quantities vanish at \(\gamma=0\).  At the clean-channel endpoint,
the nearest-center decoding argument used in
Lemma~\ref{lem:budget-regularity} yields
\begin{equation}
\label{eq:empirical-clean-information-partition}
\lim_{\gamma\to\infty}\MI(I;Y_\gamma)=\log n,
\qquad
\lim_{\gamma\to\infty}\mathcal M_\gamma
=\log n-H(I\mid R),
\qquad
\lim_{\gamma\to\infty}\overline{\mathcal J}_\gamma=H(I\mid R).
\end{equation}
The measurement therefore supplies an average head start
\(\MI(I;R)=\log n-H(I\mid R)\); reverse denoising resolves the remaining
index entropy \(H(I\mid R)\).  The class and directional analyses in
Section~5 refine this scalar accounting by identifying which distinctions
remain and where measurement information is represented.

Table~\ref{tab:information-dynamics-map} collects the exact coordinates used
throughout the solvable analysis.

\begin{table}[t]
\centering
\footnotesize
\setlength{\tabcolsep}{5pt}
\renewcommand{\arraystretch}{1.14}
\begin{tabular}{@{}
>{\raggedright\arraybackslash}p{3.45cm}
>{\raggedright\arraybackslash}p{4.35cm}
>{\raggedright\arraybackslash}p{5.55cm}@{}}
\toprule
Aspect & Exact coordinates & Role \\
\midrule
Measurement assimilation &
\(\mathcal K_t(r)\), \(\mathcal M_\gamma\),
\(\E\|g_t(X_t,R)\|^2\) &
fixed-measurement information, its measurement average, and the corresponding
force-energy rate \\
Class resolution &
\(\eta(r)\), \(\delta=(\Id_d-KA)m\), \(\widetilde\phi_r(t)\) &
measurement evidence, residual separation, and baseline-corrected commitment \\
Sample-index resolution &
\(H(I\mid r)\), \(\mathcal J_t(r)\),
\(\overline{\mathcal J}_\gamma\) &
remaining index budget and its fixed-measurement and averaged resolution \\
Directional allocation &
\(\rho_i(t)\), \(\mathcal M_\gamma(A)\) &
direction-wise force strength and operator-dependent information assimilation \\
\bottomrule
\end{tabular}
\caption{Summary of the exact coordinates used in Sections~4--5.  The rows
describe complementary aspects of conditioning and need not change at a
common noise level.}
\label{tab:information-dynamics-map}
\end{table}
\FloatBarrier

\subsection{Gaussian-Moment Analysis of Approximate Samplers}
\label{sec:sampler-implications}
\input{sections/algorithms}

%% file: sections/algorithms.tex
The exact smoothed likelihood force depends on the conditional clean-signal
law \(X_0\mid X_t=x\), which is not available for a learned prior.  To compare
common measurement updates on a common probabilistic basis, we approximate
this conditional law by a Gaussian with the same first two moments.  This
calculation separates the roles of the predictive mean, predictive covariance,
and their state dependence.

\runinhead{Local Gaussian approximation.}
Recall the prior denoiser \(m_t(x)=\E[X_0\mid X_t=x]\), and define
\begin{equation}
\label{eq:local-gaussian-moments}
C_t(x)=\Cov(X_0\mid X_t=x),
\qquad
S_t(x)=AC_t(x)A^\top+\Gamma.
\end{equation}
Under the linear Gaussian measurement model, replacing
\(X_0\mid X_t=x\) by \(\calN(m_t(x),C_t(x))\) gives the surrogate
\begin{equation}
\label{eq:local-gaussian-surrogate}
p(r\mid X_t=x)
\approx
\calN\!\left(r;Am_t(x),S_t(x)\right),
\qquad
\ell_t^{\rm G}(x;r)
:=\log\calN\!\left(r;Am_t(x),S_t(x)\right).
\end{equation}
Write \(e_t(x)=r-Am_t(x)\), and let \(D m_t(x)\) denote the Jacobian of
\(m_t\).  If \(m_t\) and \(C_t\) are differentiable, then
\begin{equation}
\label{eq:local-gaussian-full-score}
\begin{aligned}
[\nabla_x\ell_t^{\rm G}(x;r)]_j
={}&[(D m_t(x))^\top A^\top S_t(x)^{-1}e_t(x)]_j\\
&-\frac12\Tr\!\left[S_t(x)^{-1}\partial_jS_t(x)\right]\\
&+\frac12e_t(x)^\top S_t(x)^{-1}
(\partial_jS_t(x))S_t(x)^{-1}e_t(x).
\end{aligned}
\end{equation}
The first term differentiates the predictive mean, while the remaining terms
come from state dependence in the predictive covariance.  Equation
\eqref{eq:local-gaussian-full-score} is exact for the surrogate and reduces to
the exact smoothed likelihood force in the linear Gaussian model of
Proposition~\ref{thm:linear-gaussian-guidance}, where \(C_t(x)\) is independent
of \(x\).

\runinhead{Residual-gradient guidance.}
For isotropic measurement noise, \(\Gamma=\sigma_y^2\Id_{d_y}\), DPS
\citep{chung2023diffusion} differentiates the clean-space residual
\begin{equation}
\label{eq:dps-loss}
\mathcal L_t(x_t)
=
\frac{\|A\hat x_0(x_t,t)-r\|^2}{2\sigma_y^2},
\end{equation}
where \(\hat x_0\) is the sampler's denoised estimate.  In the Gaussian-moment
calculation, this amounts to differentiating the predictive mean while using
the measurement variance in place of the predictive covariance
\(S_t(x)\).  The one-dimensional calculation in
Appendix~\ref{sec:app-exact-vs-plugin-force} makes the
resulting scale difference explicit: before step-size normalization, the
plug-in force exceeds the exact force by
\(1+v_t/\sigma_y^2\), where \(v_t=\Var(X_0\mid X_t)\).  The
residual-dependent normalization in the discrete DPS implementation changes
this calibration, although the exact, plug-in, and normalized force energies
retain the same \(O(a_t^2)\) high-noise scaling in the VP model.

\runinhead{Covariance and projection approximations.}
\PiGDM \citep{song2023pigdm} combines pseudo-inverse structure with an
isotropic approximation of the predictive variance, whereas DMPS
\citep{meng2022diffusion} uses the pivot and variance obtained from an
uninformative-prior approximation.  DDNM/DDNM+ \citep{wang2023zero} follows a
different route, enforcing clean-space range/null consistency through a
projection-like update rather than approximating the posterior score.  In the
isotropic Gaussian reference, the posterior gain
\begin{equation}
\label{eq:gaussian-posterior-gain}
K_{\sigma_y}=\sigma_x^2A^\top
(\sigma_x^2AA^\top+\sigma_y^2\Id_{d_y})^{-1}
\end{equation}
satisfies
\begin{equation}
\label{eq:ddnm-limit}
K_{\sigma_y}A\to A^\dagger A,
\qquad
\Id_d-K_{\sigma_y}A\to\Id_d-A^\dagger A
\quad\text{as }\sigma_y^2\to0.
\end{equation}
Thus the hard row-space projection is the zero-noise limit of the softer
Gaussian posterior update.

The sampler configurations also differ in Jacobian treatment, step scaling,
stochasticity, and the point at which the correction is inserted.  The
comparisons in Section~\ref{sec:neural-diagnostics} consequently concern the
complete fixed configurations rather than individual design choices.

\runinhead{Trajectory observables.}
The exact quantities \(\mathcal K_t\), \(\mathcal M_\gamma\), and
\(\mathcal J_t\) cannot be evaluated for the learned FFHQ prior.  The neural
study instead applies a common frozen denoiser to each trajectory and records
two observable features: its residual from the measurement set and its
across-chain dispersion in the row and null spaces of \(A\).  The projected
total-covariance identity in
Lemma~\ref{prop:posterior-mean-dispersion-calibration} supplies the exact-model
motivation for the dispersion coordinates.  When a scalar path summary is
needed, dispersion is integrated over a common log-SNR grid to obtain its area
under the curve (AUC).  These observables preserve the timing and
subspace distinctions of the theory, but are not measured in information
units.  Section~\ref{sec:neural-diagnostics} defines their empirical estimators.

%% file: sections/experiments.tex
The exact-model studies use settings in which the posterior is available in
closed form or by finite summation.  They test the predicted high-noise
scaling and conditional information budgets, and then calibrate the
clean-estimate dispersion used in the neural experiments.
Forward noise time \(t\) is displayed in the reverse denoising direction.

\runinhead{High-noise force decay.}
Corollaries~\ref{cor:general-high-noise-energy}
and~\ref{cor:general-relative-energy} predict that both force energy and its
ratio to prior-score energy scale linearly in \(a_t^2\) at high noise.  Across
denoising, super-resolution, and inpainting, log--log fits against \(a_t^2\)
give exponents between \(1.002\) and \(1.008\), close to the predicted unit
exponent.  Prior-score energy remains near one, and the leading force-energy
coefficient agrees with
\(d^{-1}\Tr\Cov(\E[X_0\mid R])\) to relative error below
\(4\times10^{-4}\).  Appendix~\ref{sec:app-exact-details} gives the taskwise
estimates, sampling design, and energy curves.

\subsubsection{Residual Class Information}
\label{sec:diag-class-speciation}

In the conditional Gaussian mixture of Section~\ref{sec:speciation}, the
measurement may supply class evidence before the noisy state becomes
informative.  The relevant dynamic quantities are therefore the residual
coordinate
\(\kappa_t=a_t^2\delta^\top\Lambda_t^{-1}\delta\) and the excess cloning over
the static baseline \(\phi_r(\infty)\).  Figure~\ref{fig:conditional-gmm}
shows the resulting distinction.  At zero external field, \(\kappa_t=1\)
marks the local symmetric instability; a nonzero field replaces it by a
biased crossover.  Increasing the observed fraction or reducing measurement
noise suppresses only the measured component of class separation.

\begin{figure}[tb]
\centering
\includegraphics[width=0.92\linewidth]{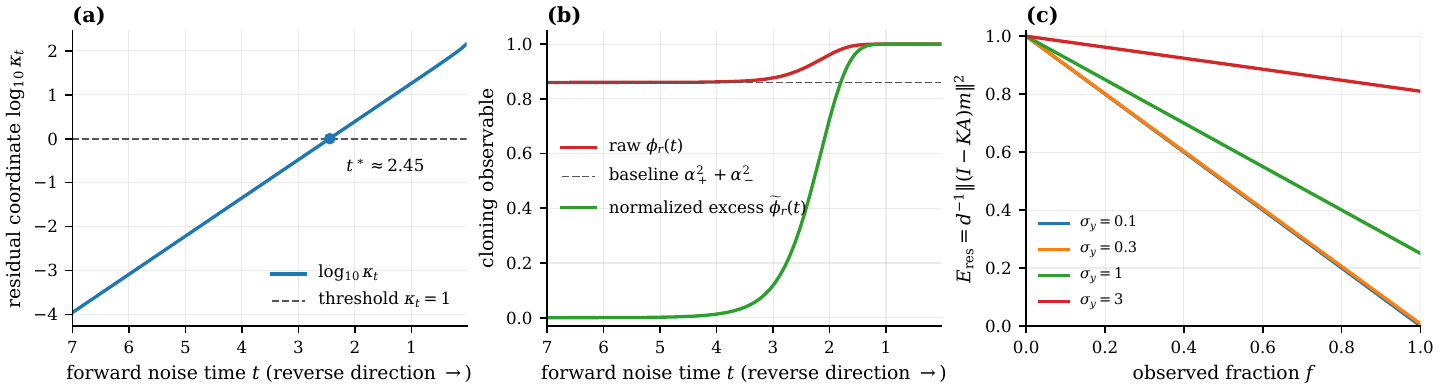}
\caption{Residual class-information diagnostics in an exact conditional
Gaussian mixture.  (a)~At zero field, the residual coordinate \(\kappa_t\)
crosses the local-instability reference \(\kappa_t=1\).
(b)~With a measurement-induced field, subtracting the high-noise cloning
baseline isolates evidence added by the noisy state.  (c)~Residual separation
\(d^{-1}\|(\Id_d-KA)m\|^2\) decreases only in observed directions.}
\label{fig:conditional-gmm}
\end{figure}

\subsubsection{Empirical Explanation Resolution}
\label{sec:diag-sample-collapse}

For a finite empirical prior, posterior collapse is decoding among
measurement-compatible explanations.  Given data \(\{x_i\}_{i=1}^n\), we
compute
\begin{equation}
\label{eq:mnist-empirical-weights}
w_i(r)\propto
\exp\left\{-\frac{\|Ax_i-r\|^2}{2\sigma_y^2}\right\}
\end{equation}
and evaluate \(H(I\mid r)\), \(H(I\mid X_t,r)\), and
\(\mathcal J_t(r)=H(I\mid r)-H(I\mid X_t,r)\) exactly by finite summation.
The experiment uses 1500 MNIST \citep{lecun1998gradient} samples, 100
in-support references, and a common assumed likelihood width \(\sigma_y=2\)
for denoising, \(2\times\) super-resolution, and central-block inpainting.
No measurement noise is sampled in this diagnostic; \(\sigma_y\) specifies
the width of the likelihood used to form the posterior weights.

Figure~\ref{fig:mnist-empirical} displays the budget change established by
Proposition~\ref{thm:posterior-collapse}.  The high-noise entropy begins at
the task-dependent value \(H(I\mid r)\), not \(\log n\), and vanishes as the
empirical components separate.  Denoising leaves the smallest compatible set
and the strongest class baseline; the lower-rank operators preserve much more
index ambiguity.  Means and standard errors are computed across references.

\begin{figure}[tb]
\centering
\includegraphics[width=0.91\linewidth]{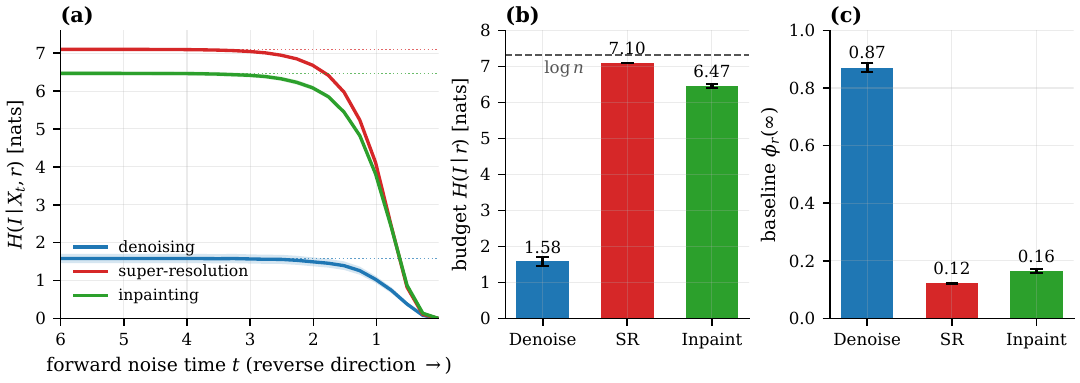}
\caption{Exact empirical-posterior diagnostics on a 1500-sample MNIST subset
(100 references; bands and bars are standard errors).
(a)~\(H(I\mid X_t,r)\) starts at the measurement-dependent budget
\(H(I\mid r)\) and decays as components separate; the colored dotted
horizontals are the corresponding mean budgets.
(b)~Posterior budgets; the gray dashed line is the unconditional reference
\(H(I)=\log n\).
(c)~Measurement-induced high-noise class-cloning baselines.}
\label{fig:mnist-empirical}
\end{figure}

\runinhead{Directional allocation.}
In an isotropic Gaussian projection model, the row-space and global force
ratios have identical time dependence and differ by the rank fraction.  For the
\(d_y/d=1/16\) geometry of \(4\times\) super-resolution this gives the exact
sixteen-fold dilution predicted by
Proposition~\ref{prop:null-space-dilution}.

\subsubsection{Information Dynamics on Common Exact Trajectories}
\label{sec:unified-trajectories}

We next record the class, index, force, and residual diagnostics on the same
exact-score conditional reverse trajectories.  For each MNIST operator, twenty
references and 32 chains follow a 400-step Euler--Maruyama discretization on
\(t\in[0.02,6]\).  Forward conditional sampling provides an independent
entropy check; across the saved checkpoints, the maximum discrepancy between
the reference-averaged forward and reverse entropy curves is \(0.12\) nats,
and the terminal reverse entropy is below \(5\times10^{-9}\) nats.  The
200/400/800-step comparison in Table~\ref{tab:app-em-stability} shows that the half-times and
residual floors are stable to the discretization.

Figure~\ref{fig:unified-trajectories} shows four complementary effects.  The
fixed-reference row-space force ratio rises from the predicted high-noise
scaling, peaks in the middle of the trajectory, and remains below
\(6\times10^{-3}\).  Its low-noise decay is specific to the finite empirical
prior: once \(X_t\) identifies an atom visited by the posterior trajectory, the
prior and posterior denoisers both converge to that atom, so
\(m_t^r(X_t)-m_t(X_t)\) vanishes.  Class entropy contracts from the static
\(H(C\mid r)\) baseline.  The denoised measurement residual rises from the
posterior-mean fit toward the compatible-pool sampling floor.  Finally,
\(\log n-H(I\mid X_t,r)\) separates into a measurement head start
\(\log n-H(I\mid r)\) and the subsequent resolution stock
\(\mathcal J_t(r)\).  The hollow circles in panel (d) are independent forward
conditional Monte Carlo estimates at selected checkpoints.  Together the four
panels show that class resolution, force transfer, and empirical collapse need
not share a universal ordering.

\begin{figure}[t]
\centering
\includegraphics[width=0.86\linewidth]{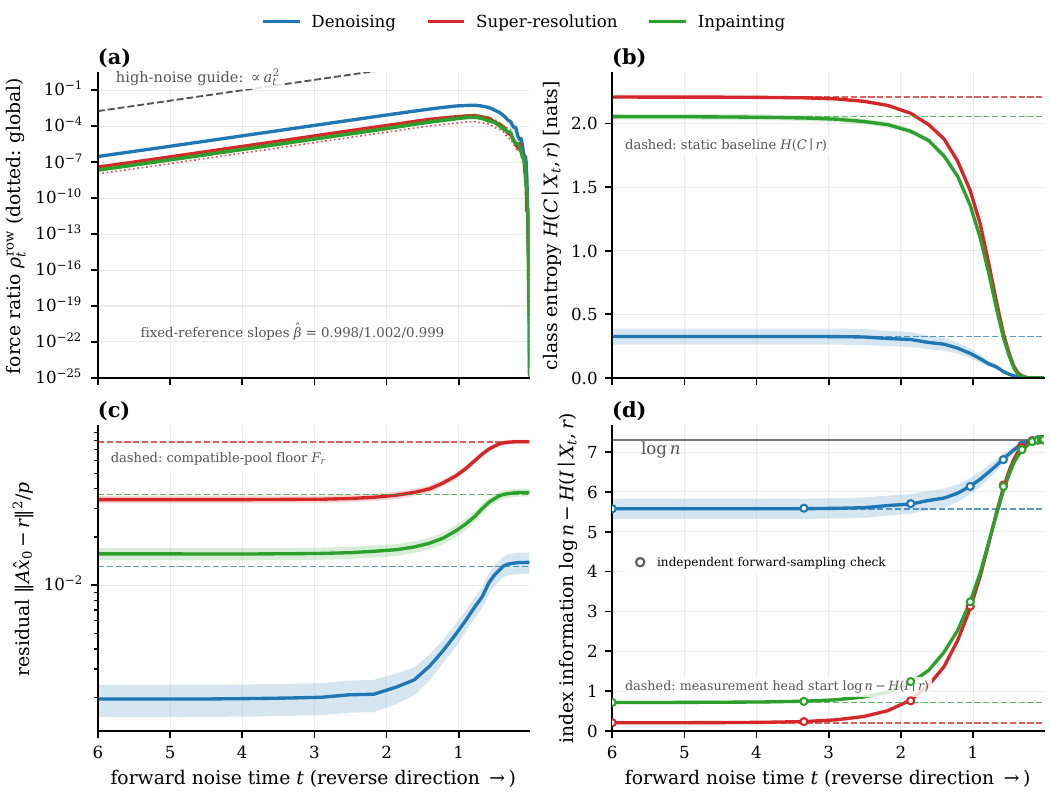}
\caption{Information diagnostics on common exact-score conditional reverse
trajectories (MNIST empirical prior; twenty references, 32 chains).
(a)~Fixed-reference row-space force ratio; dotted colored curves are the
corresponding global ratios, and the gray dashed line is a unit-slope reference
for the \(a_t^2\) high-noise scaling.
(b)~Class entropy, with static measurement baselines dashed.
(c)~Denoised residual approaching the compatible-pool sampling floor.
(d)~Index information, decomposed into the measurement head start (colored
dashes) and subsequent trajectory resolution; hollow circles show independent
forward conditional samples at every fourth saved checkpoint.}
\label{fig:unified-trajectories}
\end{figure}

\subsubsection{Calibration of Clean-Estimate Dispersion}
\label{sec:exact-proxy-calibration}

Lemma~\ref{prop:posterior-mean-dispersion-calibration} decomposes the clean
posterior covariance into uncertainty remaining after \(X_t\) and dispersion
of \(m_t^r(X_t)=\E[X_0\mid X_t,r]\) across noisy states.  We evaluate this
identity for twenty references, with 256 conditional states per reference at
each of 25 log-SNR values.  Figure~\ref{fig:exact-proxy-calibration} shows the
two components exchanging mass along reverse denoising in both row and null
coordinates.
The sum of the two Monte Carlo components agrees with the exact fixed
posterior variance with
median relative gap \(\ProxyClosureMedian\) and 95th-percentile gap
\(\ProxyClosureNinetyFive\).  Across-state clean-estimate dispersion is thus a
resolved component in the total-covariance decomposition.

\begin{figure}[t]
\centering
\includegraphics[width=0.91\linewidth]{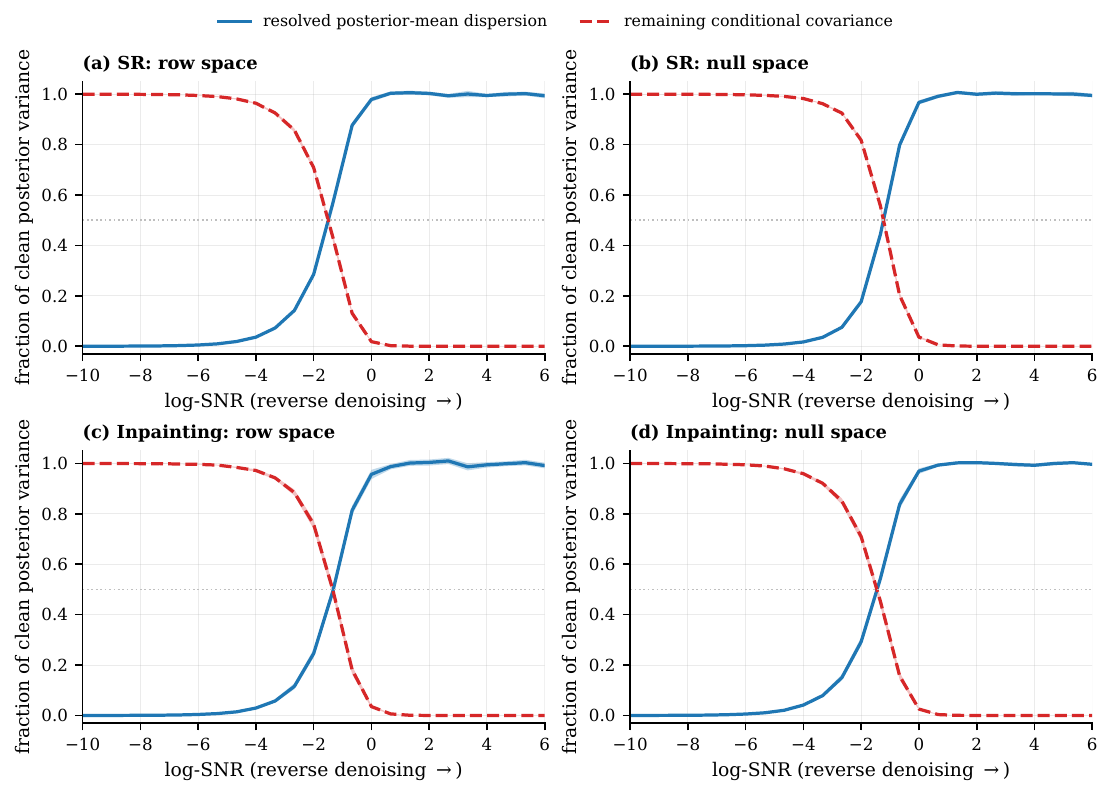}
\caption{Total-covariance calibration for the exact finite empirical posterior.
Solid curves are posterior-mean dispersion as a fraction of clean posterior
variance; dashed curves are remaining conditional covariance.
(a)--(b)~Super-resolution row and null coordinates.
(c)--(d)~Inpainting row and null coordinates.  Shaded regions are 95\%
reference-bootstrap intervals; at most checkpoints they are narrower than the
plotted curves.}
\label{fig:exact-proxy-calibration}
\end{figure}

%% file: sections/neural_trajectory.tex
For learned image priors, the exact posterior score and the information
quantities of Sections~\ref{sec:posterior-score}--\ref{sec:sampler-implications}
are generally unavailable.  We use a frozen FFHQ model as a concrete test bed
and examine two denoiser-based summaries along trajectories from fixed
approximate samplers: measurement residual and between-chain dispersion in the
measured and unmeasured subspaces.

\runinhead{Common trajectory summaries.}
All samplers use the same unconditional FFHQ checkpoint.  For chain \(k\), let
\(\widehat x_{0,t}^{(k)}\) denote the clean estimate obtained by evaluating the
frozen denoiser before the sampler-specific correction at a saved state:
\begin{equation}
\label{eq:neural-common-estimate}
\widehat x_{0,t}^{(k)}=D_{\rm frozen}(x_t^{(k)},t),
\qquad
\widetilde x_{0,t}^{(k)}
=\operatorname{clip}_{[-1,1]}\!\left(\widehat x_{0,t}^{(k)}\right).
\end{equation}
The pre-correction residual for that chain is
\begin{equation}
\label{eq:neural-common-residual}
R_{t,k}^{\rm pre}
=d_y^{-1}\|A\widehat x_{0,t}^{(k)}-r\|_2^2.
\end{equation}
For \(K\) independent chains, we also record
\begin{equation}
\label{eq:neural-subspace-variation}
D_{\rm row}^{\rm clip}(t)
 =
\frac{\Tr\widehat{\Cov}_k
\bigl(P_A\widetilde x_{0,t}^{(k)}\bigr)}{r_A},
\qquad
D_{\rm null}^{\rm clip}(t)
 =
\frac{\Tr\widehat{\Cov}_k
\bigl(P_A^\perp\widetilde x_{0,t}^{(k)}\bigr)}{d-r_A}.
\end{equation}
Here \(\widehat{\Cov}_k\) is the sample covariance with denominator \(K-1\).
The row and null summaries are defined when \(r_A>0\) and \(d-r_A>0\),
respectively.
The residual curves average \eqref{eq:neural-common-residual} over chains and
references.  The clipping and projectors in
\eqref{eq:neural-subspace-variation} are identical for all samplers.
Lemma~\ref{prop:posterior-mean-dispersion-calibration}
motivates the row/null decomposition in the exact posterior model; in the
learned setting, \eqref{eq:neural-common-residual}
and~\eqref{eq:neural-subspace-variation} are empirical trajectory summaries.

\phantomsection\label{sec:neural-consistency-transfer}
\runinhead{Sampler configurations.}
We evaluate validation-selected configurations of DPS
\citep{chung2023diffusion}, DDNM+ \citep{wang2023zero}, \PiGDM
\citep{song2023pigdm}, and DMPS \citep{meng2022diffusion}.  Because their update
rules differ in several respects, the comparisons concern the complete fixed
configurations rather than any single algorithmic component.
Appendix~\ref{sec:app-neural-protocol} gives the configurations,
common-observable protocol, and finite-chain analysis.

\phantomsection\label{sec:neural-mask-geometry}
\runinhead{Fixed-spectrum comparison.}
The exact Gaussian result in Theorem~\ref{thm:matched-spectrum-alignment}
concerns operator--prior alignment.  The following learned-model illustration
asks a narrower question: whether masks with the same spectrum can yield
different empirical subspace trajectories under a frozen image prior.  The
fixed library
contains five contiguous, three multi-hole, three structured-distributed, and
five random inpainting masks.  Each mask hides one quarter of the image
locations, so all sixteen operators have the same rank, nullity, and singular
values.

For each sampler, mask \(A\), and held-out reference image, we integrate the
null-space dispersion over a common 25-checkpoint log-SNR grid to obtain an
area under the curve (AUC).  We compare the guided AUC with that of an
unconditional prior trajectory evaluated under the same mask:
\[
Z_{\rm null}(A)
=\log\frac{\operatorname{AUC}_{\rm null}^{\rm guided}(A)}
{\operatorname{AUC}_{\rm null}^{\rm prior}(A)},
\qquad
\Delta_{\rm null}
=\overline{Z}_{\rm null}^{\rm contiguous}
-\overline{Z}_{\rm null}^{\rm random}.
\]
The prior normalization removes mask-dependent dispersion already present
without measurement guidance.  The contrast is relative rather than
absolute: \(\Delta_{\rm null}>0\) means that the guided-to-prior AUC ratio is
larger for contiguous than for random masks.

Figure~\ref{fig:fixed-library-contrasts} reports method-specific
contrasts.  In natural-log units, the null-space estimates range from
\(0.468\) to \(0.652\), and
their 95\% reference-bootstrap intervals, conditional on the fixed masks, lie
above zero.  The row-space estimates are slightly negative and much smaller in
magnitude, ranging from \(-0.060\) to \(-0.012\).  The family difference is
therefore concentrated in the null-space trajectory statistic.

\begin{figure}[t]
\centering
\includegraphics[width=0.91\linewidth]{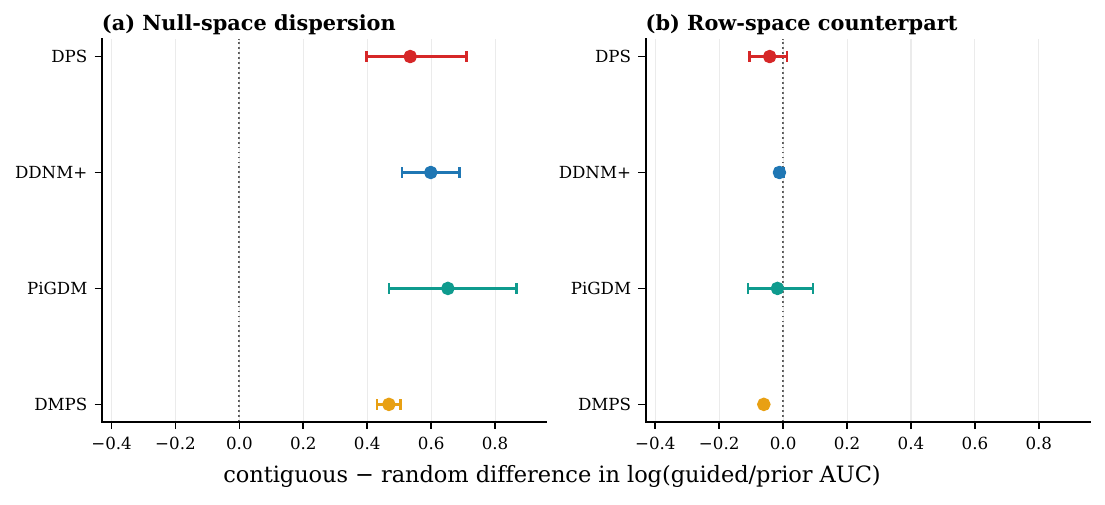}
\caption{Method-specific contrasts between five contiguous and five random
masks from the fixed library.  Points are means over held-out references; bars
are 95\% reference-bootstrap intervals conditional on the masks.  Positive
values in (a) mean that the guided-to-prior null-space AUC ratio is larger for
contiguous masks.  Panel (b) shows the corresponding row-space contrasts.}
\label{fig:fixed-library-contrasts}
\end{figure}

For every sampler, the multi-hole estimate lies between the contiguous and
random estimates, whereas the structured-distributed estimate is close to the
random value.  Thus, within this checkpoint and fixed library, masks with a
common spectrum yield different null-space dispersion AUCs after prior
normalization.
Because mask geometry and conditional task difficulty vary together, the
experiment does not identify a causal effect of geometry.
Appendix~\ref{sec:app-neural-protocol} gives the full mask library, protocol,
and sensitivity analyses.

%% file: sections/conclusion.tex
This paper develops an information-theoretic account of how linear
measurements alter diffusion reverse dynamics.  The posterior--prior score
difference is a smoothed likelihood force whose conditional energy gives the
posterior--prior entropy-dissipation rate.  Its integral, together with the
terminal marginal divergence, determines reverse-path relative entropy, while
averaging over measurements yields an I-MMSE relation.  Under finite second
moments, the joint-law force becomes weak relative to the prior score at high
noise.  The
solvable models further show
that conditioning removes class separation already explained by the
measurement, reduces the empirical explanation budget from \(\log n\) to
\(H(I\mid r)\), and makes information assimilation depend on operator--prior
alignment beyond the singular-value spectrum.

The exact-model diagnostics illustrate these quantities directly.  A separate
learned-model study finds different null-space trajectory statistics within
one frozen FFHQ checkpoint and one fixed-spectrum mask library.
Extending the analysis beyond linear measurements and testing the directional
predictions across independently trained models remain natural next steps.
Together, the results show that measurement information has both a budget and
a geometry, neither of which is captured by endpoint quality or the operator
spectrum alone.

%% file: appendix/proofs.tex
\subsection{Proof of Lemma~\ref{thm:posterior-tweedie}}

\begin{proof}
Assumption~\eqref{eq:indep-assump} gives
\(p(x\mid x_0,r)=q_t(x\mid x_0)\).  Integrating over the conditional law of
\(X_0\) and applying Bayes' rule yields
\(p(r\mid X_t=x)=\E[p(r\mid X_0)\mid X_t=x]\), proving
\eqref{eq:posterior-bayes-factorization}.  The Gaussian noising channel has
score
\[
\nabla_x\log p(x\mid x_0)=\frac{a_tx_0-x}{\Delta_t}.
\]
Applying Fisher's latent-variable score identity (Section~\ref{sec:posterior-score}) to the
posterior noised density (with latent \(X_0\), whose conditional given
\((X_t=x,R=r)\) exists by the factorization above) gives
\[
\nabla_x\log p_t(x\mid r)
=\E\Bigl[\frac{a_tX_0-x}{\Delta_t}\;\Big|\;X_t=x,R=r\Bigr]
=\frac{a_tm_t^r(x)-x}{\Delta_t},
\]
and \(\nabla_x\log p_t(x)=(a_tm_t(x)-x)/\Delta_t\) by the same argument
without conditioning on \(R\).
By \eqref{eq:posterior-score-factorization},
\(g_t(x,r)=\nabla_x\log p_t(x\mid r)-\nabla_x\log p_t(x)\).  The two
\(-x/\Delta_t\) terms cancel, giving \eqref{eq:general-denoiser-guidance}.
\end{proof}

\subsection{Proof of Theorem~\ref{thm:fixed-r-entropy-dissipation}}

\begin{proof}
Write \(q_u=p_u(\cdot\mid r)\), let \(p_u=p_u(\cdot)\), and set
\(h_u=q_u/p_u\).  Both densities solve the
same Fokker--Planck equation associated with \eqref{eq:forward-ou-sde},
\[
\partial_u f=\Delta f+\nabla\!\cdot(xf).
\]
Under the stated regularity and boundary assumptions,
\[
\frac{\dd}{\dd u}\KL(q_u\|p_u)
=\int (\partial_u q_u)\log h_u\,\dd x
-\int h_u\,\partial_u p_u\,\dd x.
\]
The two drift contributions cancel after integration by parts.  For the
diffusion terms,
\begin{align*}
\int (\Delta q_u)\log h_u\,\dd x-
\int h_u\Delta p_u\,\dd x
&=-\int \nabla q_u\cdot\nabla\log h_u\,\dd x
+\int \nabla h_u\cdot\nabla p_u\,\dd x\\
&=-\int p_u\frac{\|\nabla h_u\|^2}{h_u}\,\dd x\\
&=-\int q_u\|\nabla\log h_u\|^2\,\dd x.
\end{align*}
By Bayes' factorization, $\nabla\log h_u=g_u(\cdot,r)$, proving
\eqref{eq:fixed-r-kl-dissipation}; integration gives
\eqref{eq:fixed-r-kl-integrated}.  Finally,
\[
\E_R\mathcal K_u(R)
=\int p_R(r)\int p_u(x\mid r)
\log\frac{p_u(x\mid r)}{p_u(x)}\,\dd x\,\dd\nu(r)
=\MI(R;X_u).
\]
Averaging the integrated identity and applying Tonelli proves
\eqref{eq:joint-mutual-information-dissipation}.
\end{proof}

\subsection{Proof of Lemma~\ref{lem:budget-regularity}}

\begin{proof}
For the finite-support model, write
\[
p_u(x)=\sum_i\pi_i\varphi_{\Delta_u}(x-a_ux_i),
\qquad
q_u^r(x)=\sum_iw_i(r)\varphi_{\Delta_u}(x-a_ux_i),
\]
where \(\varphi_{\Delta}\) is the centered Gaussian density with covariance
\(\Delta\Id_d\).  For \(u>0\), both mixtures are smooth, \(p_u\) is strictly
positive, and \(q_u^r/p_u\le\max_i w_i(r)/\pi_i\).  On a compact interval
\([s,t]\subset(0,\infty)\), the Tweedie representation gives
\[
\left\|\nabla\log\frac{q_u^r}{p_u}(x)\right\|
\le
\frac{2a_u}{\Delta_u}\max_i\|x_i\|,
\]
so the relative Fisher information is bounded and integrable.  Gaussian
tails also make every integration-by-parts boundary term vanish.

If the atoms are pairwise distinct, nearest-center decoding among the
\(\{a_u x_i\}\) recovers the component index with probability tending to one as
\(u\to0^+\).  The channel \(I\mapsto X_u\) gives the upper bound
\(\KL(q_u^r\|p_u)\le\KL(w(r)\|\pi)\), while data processing through this
decoder gives the matching lower bound in the limit.  Therefore
\[
\KL(q_u^r\|p_u)\longrightarrow
\sum_iw_i(r)\log\frac{w_i(r)}{\pi_i}.
\]
As \(u\to\infty\), \(a_u\to0\) and \(\Delta_u\to1\); the two finite mixtures
converge, with Gaussian domination, to the same \(\calN(0,\Id_d)\) density.
Hence \(\KL(q_u^r\|p_u)\to0\).

In the Gaussian setting, \(p_u\) and \(q_u^r\) are nondegenerate Gaussian
densities with covariances
\(a_u^2\Sigma+\Delta_u\Id_d\) and
\(a_u^2S_r+\Delta_u\Id_d\).  Their log-density ratio is quadratic and its
gradient is affine.  Gaussian moments therefore give finite relative entropy,
integrable relative Fisher information, and vanishing boundary terms on every
compact subinterval of \((0,\infty)\).  Continuity of the Gaussian KL formula
gives the clean endpoint as \(u\to0^+\), while both laws converge to
\(\calN(0,\Id_d)\) as \(u\to\infty\).  Finally, the integrated force energy on
\([s,t]\) is finite in both settings by
\eqref{eq:fixed-r-kl-integrated}; the forward path-law KL is finite by the
Markov-kernel chain rule, and measurable time reversal preserves it.
\end{proof}

\subsection{Proof of Corollary~\ref{cor:reverse-path-kl}}

\begin{proof}
On the forward interval $[t,T]$, the posterior and prior path measures have
initial laws $q_t^r$ and $p_t$ and share the same OU transition kernel.
The chain rule for relative entropy therefore gives
\[
\KL(\mathbb Q^r_{[t,T]}\|\mathbb P_{[t,T]})
=\KL(q_t^r\|p_t)=\mathcal K_t(r).
\]
Time reversal is a measurable bijection on path space and preserves relative
entropy, so this is also the KL divergence of the two reverse path laws.
Combining with \eqref{eq:fixed-r-kl-integrated} proves
\eqref{eq:reverse-path-kl-identity}.

For a coefficient check, the reverse drifts differ by $2g_{T-u}$ and the
diffusion matrix is $\sqrt2\Id_d$.  Girsanov's formula
\citep{leonard2012girsanov} contributes
\[
\frac12\int_0^{T-t}
\E\left\|(\sqrt2\Id_d)^{-1}2g_{T-u}\right\|^2\,\dd u
=\int_t^T\E_{X_s\mid r}\|g_s\|^2\,\dd s,
\]
confirming the unit coefficient.
\end{proof}

\subsection{Proof of Corollary~\ref{cor:high-noise-path-proximity}}

\begin{proof}
Averaging \eqref{eq:reverse-path-kl-identity} over $R$ gives
$\E_R\mathcal K_t(R)=\MI(R;X_t)$.  The map
$X_t\mapsto Y_{\gamma_t}=X_t/\sqrt{\Delta_t}$ is invertible, hence
$\MI(R;X_t)=\MI(R;Y_{\gamma_t})$.  Conditional independence of
\(R\) and \(Y_{\gamma_t}\) given \(X_0\) gives the data-processing
inequality, and Gaussian
maximal entropy at fixed covariance gives
\[
\MI(X_0;Y_{\gamma_t})
\le
\frac12\log\det(\Id_d+\gamma_t\Cov(X_0))
\le
\frac{\gamma_t}{2}\Tr\Cov(X_0).
\]
Substituting $\gamma_t=a_t^2/\Delta_t$ proves
\eqref{eq:average-path-kl-bound} and the normalized high-noise bound.  Finally,
Pinsker gives $\|\mathbb Q^r-\mathbb P\|_{\rm TV}\le
\sqrt{\KL(\mathbb Q^r\|\mathbb P)/2}$ for each $r$; averaging and applying
Jensen proves \eqref{eq:average-path-tv-bound}.
\end{proof}

\subsection{Proof of Corollary~\ref{cor:integrated-head-start}}

\begin{proof}
Let $s\to0^+$ and $t\to\infty$ in
\eqref{eq:fixed-r-kl-integrated}.  Under the assumptions,
$\mathcal K_s(r)\to\KL(P_0(\cdot\mid r)\|P_0)$ and
$\mathcal K_t(r)\to0$, proving \eqref{eq:integrated-force-clean-kl}.  For a
uniform empirical prior on pairwise distinct support points,
\[
\KL(P_0(\cdot\mid r)\|P_0)
=\sum_iw_i(r)\log\frac{w_i(r)}{1/n}
=\log n-H(I\mid r),
\]
which gives \eqref{eq:integrated-force-head-start}.
\end{proof}

\subsection{Proof of Lemma~\ref{prop:joint-fisher-pythagoras}}

\begin{proof}
For almost every $x$,
\[
\E[g_t(X_t,R)\mid X_t=x]
=\int p(r\mid x)\nabla_x\log p(r\mid x)\,\dd\nu(r)
=\nabla_x\int p(r\mid x)\,\dd\nu(r)=0.
\]
Since $s_t^R=s_t^{\rm prior}+g_t$, both cross-moment matrices vanish:
\[
\E\left[s_t^{\rm prior}(X_t)g_t(X_t,R)^\top\right]
=
\E\left[g_t(X_t,R)s_t^{\rm prior}(X_t)^\top\right]=0.
\]
Expanding the outer product proves \eqref{eq:matrix-fisher-pythagoras}; taking
its trace gives \eqref{eq:joint-fisher-pythagoras}, and applying the quadratic
form $v^\top(\cdot)v$ gives \eqref{eq:directional-fisher-pythagoras}.
\end{proof}

\subsection{Proof of Lemma~\ref{lem:prior-score-lower-bound}}

\begin{proof}
For $t>0$, Gaussian smoothing gives a positive smooth density and integration
by parts yields
\[
\E[(X_t-\mu_t)^\top s_t^{\rm prior}(X_t)]=-d.
\]
Cauchy--Schwarz gives
$d^2\le\Tr\Cov(X_t)\,\E\|s_t^{\rm prior}(X_t)\|^2$.
Since $\Cov(X_t)=a_t^2\Cov(X_0)+\Delta_t\Id_d$, division by $d$ proves
\eqref{eq:automatic-prior-score-lower} and
\eqref{eq:automatic-prior-score-lower-C0}.
\end{proof}

\subsection{Proof of Corollary~\ref{cor:general-high-noise-energy}}

\begin{proof}
Evaluating \eqref{eq:general-denoiser-guidance} at \((X_t,R)\), squaring, and
averaging gives \eqref{eq:general-guidance-energy-identity}.  For the variance
bound, set
\[
U=m_t^R(X_t)=\E[X_0\mid X_t,R],
\qquad
V=m_t(X_t)=\E[X_0\mid X_t].
\]
Then \(V=\E[U\mid X_t]\).  Hence
\[
\E\langle U,V\rangle
=
\E\langle \E[U\mid X_t],V\rangle
=
\E\|V\|^2,
\]
and therefore
\[
\E\|U-V\|^2
=
\E\|U\|^2-
\E\|V\|^2.
\]
Since \(\E U=\E V=\E X_0\), the same identity in centered form is
\[
\E\|U-V\|^2
=
\E\|U-\E X_0\|^2-
  \E\|V-\E X_0\|^2
\le
\E\|U-\E X_0\|^2.
\]
Conditional expectation is an \(L^2\) contraction, so
\[
\E\|U-\E X_0\|^2
=
\E\|\E[X_0-\E X_0\mid X_t,R]\|^2
\le
\E\|X_0-\E X_0\|^2
=
\Tr\Cov(X_0).
\]
This proves \eqref{eq:general-guidance-energy-bound}; the normalized high-noise
scaling follows from \(\Tr\Cov(X_0)=O(d)\) and
\(\Delta_t\ge\Delta_0>0\).
\end{proof}

\subsection{Proof of Corollary~\ref{cor:general-relative-energy}}

\begin{proof}
Corollary~\ref{cor:general-high-noise-energy} gives
$\E\|g_t\|^2\le d a_t^2C_0/\Delta_t^2$, while
Lemma~\ref{lem:prior-score-lower-bound} gives
$\E\|s_t^{\rm prior}(X_t)\|^2\ge d/(a_t^2C_0+\Delta_t)$.
Dividing proves \eqref{eq:general-relative-energy-ratio}.  On
$\Delta_t\ge\Delta_0$, use $a_t^2C_0+\Delta_t\le C_0+1$ to obtain
\eqref{eq:general-relative-energy-ratio-window}.
\end{proof}

\subsection{Proof of Proposition~\ref{prop:force-information-rate}}
Throughout, \(\gamma=a_t^2/\Delta_t\) and
\(Y_\gamma=\sqrt{\gamma}\,X_0+\xi=X_t/\sqrt{\Delta_t}\), so conditioning on
\(Y_\gamma\) and on \(X_t\) generate the same \(\sigma\)-algebra; all mutual
informations below are finite for finite \(\gamma\) under
\(\E\|X_0\|^2<\infty\) \citep{guo2005mutual}.

\emph{Step 1 (chain rule).}  Conditional independence of \(R\) and the
forward noise given \(X_0\) gives
\(\operatorname{I}(R;Y_\gamma\mid X_0)=0\).  The two chain-rule
expansions of \(\operatorname{I}(X_0,R;Y_\gamma)\) give
\[
\operatorname{I}(R;Y_\gamma)
=
\operatorname{I}(X_0;Y_\gamma)-\operatorname{I}(X_0;Y_\gamma\mid R).
\]

\emph{Step 2 (I-MMSE, unconditional and conditional).}  By
\citet{guo2005mutual},
\(\frac{d}{d\gamma}\operatorname{I}(X_0;Y_\gamma)
=\tfrac12\,\E\|X_0-\E[X_0\mid Y_\gamma]\|^2\).
Conditionally on \(R=r\), the pair \((X_0,Y_\gamma)\) is again a Gaussian
channel whose input law is the posterior \(P(X_0\in\cdot\mid R=r)\), so
the same identity applies for each \(r\).  In integral form,
\[
\operatorname{I}(X_0;Y_\gamma\mid R=r)
=\tfrac12\int_0^\gamma \mathrm{mmse}_r(u)\,\dd u,
\quad
\mathrm{mmse}_r(u)=\E[\|X_0-\E[X_0\mid Y_u,R=r]\|^2\mid R=r].
\]
The integrand is nonnegative, nonincreasing in \(u\), and bounded by
\(\E[\|X_0\|^2\mid R=r]\); Tonelli's theorem therefore gives
\(\operatorname{I}(X_0;Y_\gamma\mid R)
=\tfrac12\int_0^\gamma\E[\mathrm{mmse}_R(u)]\,\dd u\).  The
finite-second-moment Gaussian-channel MMSE is continuous for \(u>0\) (and
right-continuous at \(0\)), so differentiation in \(\gamma\) yields
\(\frac{d}{d\gamma}\operatorname{I}(X_0;Y_\gamma\mid R)
=\tfrac12\,\E\|X_0-\E[X_0\mid Y_\gamma,R]\|^2\).

\emph{Step 3 (orthogonality).}  Write \(m=\E[X_0\mid Y_\gamma]\) and
\(m^R=\E[X_0\mid Y_\gamma,R]\).  Since \(m^R-m\) is
\((Y_\gamma,R)\)-measurable and \(\E[X_0-m^R\mid Y_\gamma,R]=0\), the cross
term vanishes and
\(\E\|X_0-m\|^2=\E\|X_0-m^R\|^2+\E\|m^R-m\|^2\).
Subtracting the two derivatives of Step 2 proves
\eqref{eq:force-information-derivative}; since \(Y_0=Z\) is independent of
\((X_0,R)\), \(\operatorname{I}(R;Y_0)=0\), and integration proves
\eqref{eq:force-information-integral}.

\emph{Step 4 (Tweedie).}  Lemma~\ref{thm:posterior-tweedie} gives
\[
g_t(X_t,R)=\frac{a_t}{\Delta_t}(m^R-m),
\qquad
\E\|g_t\|^2
=\frac{a_t^2}{\Delta_t^2}\E\|m^R-m\|^2.
\]
Substituting \eqref{eq:force-information-derivative} proves
\eqref{eq:force-energy-information-rate}.

\emph{Zero-SNR limit.}  The minimum mean-square error of a Gaussian channel
is continuous and nonincreasing in \(\gamma\), with value \(\Tr\Cov(X_0)\)
at \(\gamma=0\) \citep{guo2005mutual}; by the orthogonality of Step 3,
\(\E\|m-\E X_0\|^2=\Tr\Cov(X_0)-\E\|X_0-m\|^2\to0\) as
\(\gamma\to0^+\), and the conditional version gives
\(\E\|m^R-\E[X_0\mid R]\|^2\to0\).  Hence
\(\tfrac12\E\|m^R-m\|^2\to\tfrac12\Tr\Cov(\E[X_0\mid R])\), the slope
\eqref{eq:initial-measurement-information-slope} quoted after the
proposition.

\runinhead{Gaussian conditional mutual information
(equation~\eqref{eq:gaussian-conditional-information}).}
The Gaussian conditional mutual information is the difference between output
entropy and conditional output entropy:
\begin{align*}
\MI_{\rm G}(X_0;X_t\mid R=r)
&= h_{\rm G}(X_t\mid r)-h(X_t\mid X_0,r)\\
&= \frac12\log\det\!\left(2\pi e\,(a_t^2S_r+\Delta_t\Id_d)\right)
   -\frac d2\log(2\pi e\Delta_t)\\
&= \frac12\log\det\!\left(
   \Id_d+\frac{a_t^2}{\Delta_t}S_r\right).
\end{align*}

\runinhead{Score-energy convergence for bounded empirical priors.}
Let the empirical prior be supported in a ball of radius \(B\), and write the
prior score as \(s_t(x)=(a_t m_t(x)-x)/\Delta_t\) with
\(m_t(x)=\E[X_0\mid X_t=x]\), so \(\|m_t\|\le B\) pointwise.  Then
\[
d^{-1}\E\|a_tm_t(X_t)-X_t\|^2
=
\Delta_t+a_t^2\,d^{-1}\E\|X_0\|^2
-2a_t\,d^{-1}\E\langle m_t(X_t),X_t\rangle
+a_t^2\,d^{-1}\E\|m_t(X_t)\|^2,
\]
using \(\E\|X_t\|^2=a_t^2\E\|X_0\|^2+d\Delta_t\).  The absolute
contribution of the last three terms is at most
\[
\frac{a_tB}{d}
\left(2\sqrt{a_t^2B^2+d\Delta_t}+2a_tB\right).
\]
Consequently,
\[
d^{-1}\E\|s_t\|^2
=
\Delta_t^{-1}
+O\!\left(\frac{a_tB(B+\sqrt d)}{d\Delta_t^2}\right)
\longrightarrow \Delta_t^{-1}
\quad\text{as }a_t\to0,
\]
which tends to \(1\) under the schedule \eqref{eq:forward-marginal}.

\subsection{Proof of Proposition~\ref{thm:linear-gaussian-guidance}}

\begin{proof}
The joint vector \((X_t,R)\) is Gaussian with zero mean.  Its covariance
blocks are
\[
\Cov(X_t)=a_t^2\Sigma+\Delta_t\Id_d=B_t,
\qquad
\Cov(R)=A\Sigma A^\top+\Gamma=Q,
\qquad
\Cov(R,X_t)=a_tA\Sigma.
\]
Gaussian conditioning therefore gives
\[
\E[R\mid X_t=x]=a_tA\Sigma B_t^{-1}x=M_tx,
\qquad
\Cov(R\mid X_t)=Q-a_t^2A\Sigma B_t^{-1}\Sigma A^\top
=\Sigma_{R\mid t},
\]
and hence
\(\nabla_x\log p(r\mid X_t=x)
=M_t^\top\Sigma_{R\mid t}^{-1}(r-M_tx)\).
Under the joint law the residual \(R-M_tX_t\) has covariance
\(\Sigma_{R\mid t}\), hence
\[
\E\|\nabla_x\log p(R\mid X_t)\|^2
=\operatorname{Tr}\bigl(M_t^\top\Sigma_{R\mid t}^{-1}
\Sigma_{R\mid t}\Sigma_{R\mid t}^{-1}M_t\bigr)
=\operatorname{Tr}\bigl(M_t^\top\Sigma_{R\mid t}^{-1}M_t\bigr).
\]
It remains to check the high-noise scaling.  Since \(X_0\mid X_t\) has
covariance \(\Sigma-a_t^2\Sigma B_t^{-1}\Sigma\),
\(\Sigma_{R\mid t}=A(\Sigma-a_t^2\Sigma B_t^{-1}\Sigma)A^\top+\Gamma
\succeq\Gamma\), so \(\|\Sigma_{R\mid t}^{-1}\|_{\rm op}\le\gamma_0^{-1}\);
for \(a_t\) small enough that \(\Delta_t\ge1/2\),
\(B_t\succeq\Delta_t\Id_d\) gives \(\|B_t^{-1}\|_{\rm op}\le2\) and
\(\|M_t\|_F\le2a_t\|A\Sigma\|_F\).  Hence
\[
\frac1d\E\|\nabla_x\log p(R\mid X_t)\|^2
=\frac1d\operatorname{Tr}\bigl(M_t^\top\Sigma_{R\mid t}^{-1}M_t\bigr)
\le\frac{\gamma_0^{-1}}{d}\|M_t\|_F^2
\le4\gamma_0^{-1}a_t^2\,\frac1d\|A\Sigma\|_F^2=O(a_t^2).
\]
\end{proof}

\subsection{Proof of Corollary~\ref{cor:relative-energy}}

\begin{proof}
For a Gaussian prior the noised marginal is \(\calN(0,B_t)\), so the score
is exactly \(-B_t^{-1}x\) and
\(\E\|s_t(X_t)\|^2=\operatorname{Tr}(B_t^{-2}B_t)=\operatorname{Tr}(B_t^{-1})\),
which gives \eqref{eq:prior-score-energy}.  For the likelihood force, the
fixed-dimension expansions
\(B_t^{-1}=\Delta_t^{-1}\Id_d-a_t^2\Delta_t^{-2}\Sigma+O(a_t^4)\),
\(M_t=(a_t/\Delta_t)A\Sigma+O(a_t^3)\), and
\(\Sigma_{R\mid t}^{-1}=Q^{-1}+O(a_t^2)\) (using \(Q\succ0\)) give
\[
\frac1d\E\|g_t(X_t,R)\|^2
=\frac1d\Tr\bigl(M_t^\top\Sigma_{R\mid t}^{-1}M_t\bigr)
=\frac{a_t^2}{\Delta_t^2}\frac1d
\Tr\bigl(\Sigma A^\top Q^{-1}A\Sigma\bigr)+O(a_t^4),
\]
which is \eqref{eq:likelihood-energy-leading}.  The denominator
\(d^{-1}\Tr(B_t^{-1})\) converges to \(\Delta_\infty^{-1}\), so the ratio is
\(O(a_t^2)\).
\end{proof}

\subsection{Proofs for the Conditional Gaussian-Mixture Results}

\begin{proof}[Class-conditional posterior]
Conditioned on \(C=c\), the relevant Gaussian quantities are
\[
X_0\sim\calN(cm,\Sigma),
\qquad
R\sim\calN(Acm,Q),
\qquad
Q=A\Sigma A^\top+\Gamma,
\]
with \(\Cov(X_0,R\mid C=c)=\Sigma A^\top\).  Gaussian conditioning with
\(K=\Sigma A^\top Q^{-1}\) gives
\[
E[X_0\mid r,C=c]=cm+K(r-Acm)=Kr+c\delta,
\qquad
\Cov(X_0\mid r,C=c)=\Sigma-KA\Sigma=S.
\]
\end{proof}

\begin{proof}[Posterior class weights]
With equal class priors, \(P(C=c\mid r)\propto p(r\mid C=c)
\propto\exp\{-\tfrac12(r-Acm)^\top Q^{-1}(r-Acm)\}\), so
\[
\log\frac{P(C=+1\mid r)}{P(C=-1\mid r)}
=2(Am)^\top Q^{-1}r=2\eta(r),
\qquad
\alpha_\pm(r)=\frac{e^{\pm\eta(r)}}{2\cosh\eta(r)}.
\]
\end{proof}

\begin{proof}[Noised posterior law]
By the class-conditional calculation, \(X_0\mid r,C=c=Kr+c\delta+w\) with
\(w\sim\calN(0,S)\) (the letter \(\eta\) is reserved for the field).
Then \(X_t=a_tKr+c\,a_t\delta+a_tw+\sqrt{\Delta_t}\xi\) with
\(a_tw+\sqrt{\Delta_t}\xi\sim\calN(0,a_t^2S+\Delta_t\Id_d)=\calN(0,\Lambda_t)\),
so in the shifted coordinate \(Z_t=X_t-a_tKr\),
\(\widetilde p_t(z\mid r)=\alpha_+\calN(z;a_t\delta,\Lambda_t)
+\alpha_-\calN(z;-a_t\delta,\Lambda_t)\).
Expanding the two exponents, the \(c\)-independent terms factor out and
\[
\widetilde p_t(z\mid r)\propto e^{-z^\top\Lambda_t^{-1}z/2}
\bigl[\alpha_+e^{a_t\delta^\top\Lambda_t^{-1}z}
+\alpha_-e^{-a_t\delta^\top\Lambda_t^{-1}z}\bigr]
\propto e^{-z^\top\Lambda_t^{-1}z/2}
\cosh\{\eta(r)+a_t\delta^\top\Lambda_t^{-1}z\},
\]
using \(\alpha_\pm=e^{\pm\eta}/(2\cosh\eta)\).
\end{proof}

\begin{proof}[Posterior score and class evidence]
Differentiating
\[
\log \widetilde p_t(z\mid r)
=\mathrm{const}-\frac12 z^\top\Lambda_t^{-1}z
+\log\cosh\!\left(\eta+a_t\delta^\top\Lambda_t^{-1}z\right)
\]
gives
\eqref{eq:conditional-gmm-score}.  For the class posterior,
\[
\frac{P(C=+1\mid z,r)}{P(C=-1\mid z,r)}
=\frac{\alpha_+\calN(z;a_t\delta,\Lambda_t)}
{\alpha_-\calN(z;-a_t\delta,\Lambda_t)}.
\]
Taking logarithms
contributes \(2\eta\) from the weights and
\(2a_t\delta^\top\Lambda_t^{-1}z\) from the Gaussian ratio, which is
\eqref{eq:posterior-class-evidence}.
\end{proof}

\subsection{Proof of Proposition~\ref{prop:posterior-speciation} and Corollary~\ref{cor:posterior-speciation-coordinate}}

\begin{proof}[Cloning baseline]
As \(t\to\infty\), \(a_t\to0\) and \(\Lambda_t\to\Id_d\).  Therefore
\[
\calN(z;a_t\delta,\Lambda_t)
\quad\text{and}\quad
\calN(z;-a_t\delta,\Lambda_t)
\]
converge to the same density.  Hence \(Z_t\) becomes conditionally independent
of \(C\) given \(r\), and
\[
P(C=c\mid Z_t,r)\to P(C=c\mid r)=\alpha_c.
\]
Thus
\[
\phi_r(t)
=
\E_{Z_t\mid r}\sum_cP(C=c\mid Z_t,r)^2
\to
\sum_c\alpha_c^2.
\]
\end{proof}

\begin{proof}[One-dimensional cloning integral]
Whiten the noised posterior by setting
\[
Y=\Lambda_t^{-1/2}Z_t,
\qquad
Y\mid C=c,r\sim\calN(cu,\Id_d),
\qquad
u=a_t\Lambda_t^{-1/2}\delta,
\qquad
\|u\|^2=\kappa_t.
\]
For \(\kappa_t>0\), set \(e=u/\sqrt{\kappa_t}\) and \(W=e^\top Y\),
so \(W\mid C=\pm1,r\sim\calN(\pm\sqrt{\kappa_t},1)\).  Coordinates
orthogonal to \(e\) have identical standard normal laws under both classes
and carry no class information, so the class posterior depends only on
\(W\): with \(G_\pm(s)=\calN(s;\pm\sqrt{\kappa_t},1)\) and
\(f=\alpha_+G_++\alpha_-G_-\),
\(P(C=\pm1\mid s,r)=\alpha_\pm G_\pm(s)/f(s)\), and
\[
\phi_r(t)
=\int\Bigl[\Bigl(\tfrac{\alpha_+G_+}{f}\Bigr)^2
+\Bigl(\tfrac{\alpha_-G_-}{f}\Bigr)^2\Bigr]f\,\dd s
=\int\frac{\alpha_+^2G_+^2+\alpha_-^2G_-^2}{\alpha_+G_++\alpha_-G_-}\,\dd s.
\]
The \(\kappa_t=0\) case follows by continuity.
\end{proof}

\begin{proof}[Zero-field local instability]
When \(\eta=0\),
\[
\ell_t(z)=\log \widetilde p_t(z\mid r)
=
\mathrm{const}
-\frac12z^\top \Lambda_t^{-1}z
+\log\cosh(a_t\delta^\top \Lambda_t^{-1}z).
\]
The gradient is
\[
\nabla\ell_t(z)
=
-\Lambda_t^{-1}z
+a_t\Lambda_t^{-1}\delta
\tanh(a_t\delta^\top \Lambda_t^{-1}z).
\]
Thus \(\nabla\ell_t(0)=0\).  The Hessian is
\[
\nabla^2\ell_t(z)
=
-\Lambda_t^{-1}
+
a_t^2
\sech^2(a_t\delta^\top \Lambda_t^{-1}z)
\Lambda_t^{-1}\delta\delta^\top \Lambda_t^{-1}.
\]
At \(z=0\),
\[
\nabla^2\ell_t(0)
=
-\Lambda_t^{-1}
+
a_t^2\Lambda_t^{-1}\delta\delta^\top \Lambda_t^{-1}.
\]
Conjugating by \(\Lambda_t^{1/2}\),
\[
\Lambda_t^{1/2}\nabla^2\ell_t(0)\Lambda_t^{1/2}
=
-\Id_d+a_t^2\Lambda_t^{-1/2}\delta\delta^\top \Lambda_t^{-1/2}.
\]
The rank-one matrix on the right has one nonzero eigenvalue
\[
a_t^2\delta^\top \Lambda_t^{-1}\delta=\kappa_t.
\]
Therefore the Hessian of \(\log p_t(\cdot\mid r)\) at the origin is negative
definite when \(\kappa_t<1\), singular when \(\kappa_t=1\), and has a positive
direction when \(\kappa_t>1\).  The boundary case \(\kappa_t=1\) is settled by
the quartic term.  In the whitened coordinate \(y=\Lambda_t^{-1/2}z\), the
log-density is separable,
\[
\ell_t(y)=\mathrm{const}-\tfrac12\|y_\perp\|^2-\tfrac12 s^2
+\log\cosh(\sqrt{\kappa_t}\,s),
\qquad s=e^\top y,
\]
with \(e=\Lambda_t^{-1/2}\delta/\|\Lambda_t^{-1/2}\delta\|\) and \(y_\perp\)
the orthogonal complement.  Since
\(\log\cosh u=u^2/2-u^4/12+O(u^6)\), at \(\kappa_t=1\) the profile along \(e\)
is \(-s^4/12+O(s^6)\), so the origin is still a strict local maximum.  For
\(\kappa_t>1\) the profile along \(e\) is
\((\kappa_t-1)s^2/2+O(s^4)>0\) near the origin, while every orthogonal
direction retains negative curvature.  Hence, for \(\kappa_t>1\), the origin
is a local minimum when \(d=1\) and a saddle when \(d>1\).  It is a local
maximum of \(\log \widetilde p_t(\cdot\mid r)\) if and only if
\(\kappa_t\le1\).
To verify monotonicity, diagonalize
\(S=O\diag(\sigma_j)O^\top\), write \(\bar\delta=O^\top\delta\), and set
\(u=a_t^2\).  Under the VP schedule,
\[
\kappa_t
=
\sum_j\frac{u\bar\delta_j^2}{1+u(\sigma_j-1)},
\qquad
\frac{\dd\kappa_t}{\dd u}
=
\sum_j\frac{\bar\delta_j^2}{[1+u(\sigma_j-1)]^2}.
\]
Thus \(\kappa_t\) is nondecreasing along the reverse trajectory and strictly
increasing whenever \(\delta\neq0\); the double well forms as \(\kappa_t\)
first crosses the threshold \(\kappa_t=1\).  In the degenerate
case \(\delta=0\) the measurement resolves the class separation completely,
\(\kappa_t\equiv0\), and no instability occurs at any time.
\end{proof}

\begin{proof}
With \(\Sigma=\sigma_x^2\Id_d\), \(\Gamma=\sigma_y^2\Id_{d_y}\),
\(A=U\diag(s_i)V^\top\), and
\(K=\sigma_x^2A^\top(\sigma_x^2AA^\top+\sigma_y^2\Id_{d_y})^{-1}\), each right
singular vector satisfies
\(KAv_i=\{\sigma_x^2s_i^2/(\sigma_y^2+\sigma_x^2s_i^2)\}v_i\), hence
\((\Id_d-KA)v_i=\{\sigma_y^2/(\sigma_y^2+\sigma_x^2s_i^2)\}v_i\); writing
\(m=\sum_im_iv_i\) gives
\(\delta_i=\sigma_y^2m_i/(\sigma_y^2+\sigma_x^2s_i^2)\).  Observed
directions with large \(s_i\) are suppressed in \(\delta\); null
directions are unchanged.
\end{proof}

\subsection{Proof of Proposition~\ref{thm:posterior-collapse}}

\begin{proof}[Posterior responsibility formula]
Bayes' rule gives \(P(I=i\mid x,r)\propto P(I=i\mid r)p(x\mid I=i,r)\);
the first factor is \(w_i(r)\), and given \(I=i\) we have \(X_0=x_i\),
so \(X_t\mid I=i,r\sim\calN(a_tx_i,\Delta_t\Id_d)\) and
\(P(I=i\mid x,r)\propto w_i(r)\exp\{-\|x-a_tx_i\|^2/(2\Delta_t)\}\).
\end{proof}

\begin{proof}[Entropy identity]
Expand the mixed joint entropy two ways:
\(h(X_t,I\mid r)=H(I\mid r)+h(X_t\mid I,r)
=h(X_t\mid r)+H(I\mid X_t,r)\).
Since \(X_t\mid I=i,r\sim\calN(a_tx_i,\Delta_t\Id_d)\) for every \(i\),
\(h(X_t\mid I,r)=\tfrac d2\log(2\pi e\Delta_t)\) independently of
\(i\); equating the two decompositions gives
\eqref{eq:posterior-collapse-entropy-identity}.
\end{proof}

\subsection{Derivation of the Gaussian scale in Remark~\ref{cor:collapse-time}}

\begin{proof}[Derivation]
Under the Gaussian approximation, the standard linear-channel formula gives
\[
\mathcal I_t^{\rm G}(r)
=
\frac12\log\det\left(\Id_d+\frac{a_t^2}{\Delta_t}S_r\right).
\]
For \(S_r=s_r\Id_d\) and \(\alpha_r=H(I\mid r)/d\), the fractional matching
equation \(\mathcal I_{t_q^{\rm G}}^{\rm G}(r)=qH(I\mid r)\) becomes
\[
q\alpha_r
=
\frac12\log\left(
1+\frac{s_re^{-2t_q^{\rm G}}}{1-e^{-2t_q^{\rm G}}}
\right).
\]
Solving for \(t_q^{\rm G}\) yields
\eqref{eq:isotropic-fractional-resolution-time}.
\end{proof}

\subsection{Proof of Theorem~\ref{prop:spectral-consistency}}

\begin{proof}
In the common eigenbasis, coordinates are independent with
\[
\widetilde X_{t,i}=a_t\widetilde X_{0,i}+\sqrt{\Delta_t}\,\xi_i,
\qquad
\widetilde R_i=s_i\widetilde X_{0,i}+\varepsilon_i,
\qquad
\widetilde X_{0,i}\sim\calN(0,\lambda_i).
\]
The noised prior variance is $B_{i,t}$, so the prior-score energy is
$B_{i,t}^{-1}$.  Scalar Gaussian conditioning gives
\[
\widetilde X_{0,i}\mid \widetilde X_{t,i}
\sim
\calN\left(\frac{a_t\lambda_i}{B_{i,t}}\widetilde X_{t,i},
\;v_{i,t}\right).
\]
Hence
\[
\widetilde R_i\mid \widetilde X_{t,i}
\sim
\calN\left(c_{i,t}\widetilde X_{t,i},
\;s_i^2v_{i,t}+\sigma_y^2\right),
\qquad
c_{i,t}=\frac{a_ts_i\lambda_i}{B_{i,t}},
\]
and
\[
g_{t,i}(x,r)
=\frac{c_{i,t}(\widetilde r_i-c_{i,t}\widetilde x_i)}
{s_i^2v_{i,t}+\sigma_y^2}.
\]
The regression residual has variance $s_i^2v_{i,t}+\sigma_y^2$, which proves
\eqref{eq:anisotropic-spectral-energies} and the first expression in
\eqref{eq:anisotropic-rho}.

Conditioning the clean coordinate on $R_i$ gives posterior variance
$\tau_i^2$, so the noised posterior variance is $\widetilde B_{i,t}$.  The
joint directional Fisher Pythagoras therefore gives
\[
\E|g_{t,i}|^2
=\widetilde B_{i,t}^{-1}-B_{i,t}^{-1},
\]
which proves the second expression in \eqref{eq:anisotropic-rho}.

For the consequences, write $u=a_t^2$ under the VP schedule.  Then
$B_{i,t}=1+u(\lambda_i-1)$ and
$\widetilde B_{i,t}=1+u(\tau_i^2-1)$.  Since
$\lambda_i>\tau_i^2$ when $s_i>0$,
\[
\frac{\dd}{\dd u}\frac{B_{i,t}}{\widetilde B_{i,t}}
=\frac{\lambda_i-\tau_i^2}{\widetilde B_{i,t}^2}>0.
\]
This proves reverse-time monotonicity.  The high- and low-noise limits follow
directly from \eqref{eq:anisotropic-rho}.  The equality $\rho_i=1$ is
exactly $B_{i,t}/\widetilde B_{i,t}=2$, giving the variance-halving statement;
strict monotonicity shows that an interior crossing is unique and occurs iff
$s_i^2\lambda_i>\sigma_y^2$.  At equality, the ratio approaches one only at
the clean endpoint.

Finally, the global and row ratios have the same numerator because
$g_{t,i}=0$ for $s_i=0$.  Their denominators are respectively
$\sum_iB_{i,t}^{-1}$ and $\sum_{i:s_i>0}B_{i,t}^{-1}$, which proves
\eqref{eq:anisotropic-weighted-dilution}.
\end{proof}

\subsection{Proof of Theorem~\ref{thm:matched-spectrum-alignment}}

\begin{proof}
For the Gaussian channel $Y_\gamma=\sqrt\gamma X_0+Z$, the conditional
covariance of $X_0$ given $Y_\gamma$ is
$C_\gamma=(\Sigma^{-1}+\gamma\Id_d)^{-1}$.  Therefore
\[
\Cov(R_A)=A\Sigma A^\top+\sigma_y^2\Id_{d_y},
\qquad
\Cov(R_A\mid Y_\gamma)=AC_\gamma A^\top+\sigma_y^2\Id_{d_y},
\]
and the Gaussian mutual-information formula gives
\eqref{eq:orientation-assimilation-general}.

For the clean-information bound, use $AA^\top=\Id_{d_y}$ to write
\[
2\MI(X_0;R_A)
=
\log\det\!\left[A(\Id_d+\sigma_y^{-2}\Sigma)A^\top\right].
\]
Poincar\'e separation for the compression of the positive-definite matrix \citep{bhatia1997matrix}
$\Id_d+\sigma_y^{-2}\Sigma$ bounds the determinant by the products of its
smallest and largest $d_y$ eigenvalues, proving
\eqref{eq:orientation-clean-information-bounds}.  The matrix determinant lemma
also gives
\[
\frac{\det\Sigma}{\det\Sigma_A^{\rm post}}
=
\det\!\left(\Id_d+\sigma_y^{-2}\Sigma^{1/2}A^\top A\Sigma^{1/2}\right)
=
\det\!\left(\Id_{d_y}+\sigma_y^{-2}A\Sigma A^\top\right),
\]
which is \eqref{eq:orientation-volume-contraction}.

For the assimilation bound, set
\[
\mathsf V_0=\sigma_y^2\Id_d+\Sigma,
\qquad
\mathsf V_\gamma=\sigma_y^2\Id_d+C_\gamma,
\]
and define
\[
Q=\mathsf V_\gamma^{1/2}A^\top(A\mathsf V_\gamma A^\top)^{-1/2}.
\]
Then $Q^\top Q=\Id_{d_y}$ and
\[
\frac{\det(A\mathsf V_0 A^\top)}{\det(A\mathsf V_\gamma A^\top)}
=
\det\!\left[Q^\top
\mathsf V_\gamma^{-1/2}\mathsf V_0 \mathsf V_\gamma^{-1/2}Q\right].
\]
Because $C_\gamma$ is a spectral function of $\Sigma$, the matrix in the
compression has eigenvalues
\[
\ell_i(\gamma)
=
\frac{\sigma_y^2+\lambda_i}
{\sigma_y^2+\lambda_i/(1+\gamma\lambda_i)}
=
\exp\{2\psi_\gamma(\lambda_i)\}.
\]
A second application of Poincar\'e separation \citep{bhatia1997matrix} bounds the determinant of the
compression by the products of the smallest and largest $d_y$ values
$\ell_i(\gamma)$, which proves
\eqref{eq:orientation-assimilation-bounds}.  Finally,
\[
\psi_\gamma'(\lambda)
=
\frac{\gamma\lambda(2\sigma_y^2+\lambda+
\gamma\sigma_y^2\lambda)}
{2(\sigma_y^2+\lambda)(1+\gamma\lambda)
[\sigma_y^2(1+\gamma\lambda)+\lambda]}>0
\]
for $\gamma,\lambda>0$.  The equality cases are the invariant subspaces
spanned by the corresponding extremal eigenvectors.
\end{proof}

\subsection{Proof of Corollary~\ref{cor:rank-one-orientation-contrast}}

\begin{proof}
For $A_i=v_i^\top$, equation
\eqref{eq:orientation-assimilation-general} reduces to
$\mathcal M_\gamma(A_i)=\psi_\gamma(\lambda_i)$.  Strict monotonicity of
$\psi_\gamma$ proves \eqref{eq:rank-one-orientation-contrast}.  The clean
mutual-information inequality follows because
$\tfrac12\log(1+\lambda/\sigma_y^2)$ is strictly increasing, and the posterior
covariance-volume ordering follows from
\eqref{eq:orientation-volume-contraction}.
\end{proof}

\subsection{Proof of Proposition~\ref{prop:null-space-dilution}}

\begin{proof}
By hypothesis~(i), \(\E_{X_t\mid r}\|g_t\|^2
=\E_{X_t\mid r}\|P_Ag_t\|^2\), so the global and row ratios have the same
numerator.  By hypothesis~(ii), the denominators are
\[
\E_{X_t\mid r}\|s^{\rm prior}_t\|^2=d\,e_t(r),
\qquad
\sum_{i\,\in\,\text{row}}\E_{X_t\mid r}|s^{\rm prior}_{t,i}|^2=r_A\,e_t(r).
\]
Dividing gives \eqref{eq:global-row-dilution}.
\end{proof}

\subsection{Proof of Lemma~\ref{prop:posterior-mean-dispersion-calibration}}

\begin{proof}
Condition on $R=r$ and apply the law of total covariance with the intermediate
sigma-field generated by $X_t$:
\[
\Cov(X_0\mid r)
=
\E[\Cov(X_0\mid X_t,r)\mid r]
+
\Cov(\E[X_0\mid X_t,r]\mid r).
\]
Multiplying by $P$ on both sides and taking traces proves
\eqref{eq:projected-total-covariance}.  At zero channel SNR,
$X_t$ becomes independent of $X_0$ conditionally on $r$, so
$m_t^r(X_t)$ converges in $L^2$ to $\E[X_0\mid r]$ and the between-state term
vanishes.  As the channel SNR tends to infinity, the Gaussian observation
reveals $X_0$ in mean square, so
$\E\Tr\Cov(X_0\mid X_t,r)\to0$; the first term vanishes and the second tends to
the projected posterior covariance.
\end{proof}

%% file: appendix/experiments.tex
This appendix gives the experimental definitions and protocols used in
Section~\ref{sec:diagnostics}.  The exact-model calculations are reproducible
from finite sums or closed-form Gaussian expressions.  The FFHQ experiments
use a fixed pretrained score model and saved multi-chain trajectories.  The
accompanying code release contains the analysis scripts, configurations, and
summary files used to regenerate the reported tables and figures.

\subsection{Exact-Model Experiments}
\label{sec:app-exact-details}

\runinhead{Empirical posterior on MNIST.}
The empirical calculations use the first 1500 Keras MNIST training images,
normalized to \([0,1]\), with 100 in-support references.  Denoising uses the
identity operator, super-resolution uses \(2\times\) average pooling, and
inpainting removes the central \(14\times14\) block.  The pseudo-inverse of
average pooling is nearest-neighbor patch upsampling, so
\(AA^\dagger=\Id_{d_y}\).

For a reference measurement \(r=Ax_\star\), the Gaussian likelihood of width
\(\sigma_y\) gives
\[
\ell_i(r)=-\frac{\|Ax_i-r\|^2}{2\sigma_y^2},
\qquad
w_i(r)=\frac{e^{\ell_i(r)}}{\sum_j e^{\ell_j(r)}}.
\]
The posterior index entropy and class probabilities are
\[
H(I\mid r)=-\sum_i w_i(r)\log w_i(r),
\qquad
P(C=c\mid r)=\sum_{i:y_i=c}w_i(r).
\]
At noise level \(t\), the conditional index responsibilities satisfy
\[
P(I=i\mid x,r)
\propto
w_i(r)\exp\!\left(-\frac{\|x-a_tx_i\|^2}{2\Delta_t}\right),
\]
and their entropy is averaged over exact draws from \(p_t(\cdot\mid r)\).
For the ten MNIST classes, let
\(p_c(x,r)=\sum_{i:y_i=c}P(I=i\mid x,r)\).  The cloning diagnostic is
\(\phi_r(t)=\E_{X_t\mid r}\sum_{c=0}^{9}p_c(X_t,r)^2\), with high-noise
baseline \(\phi_r(\infty)=\sum_{c=0}^{9}P(C=c\mid r)^2\).
Figures~\ref{fig:mnist-empirical} and~\ref{fig:unified-trajectories} use
\(\sigma_y=2\).  Each entropy curve uses 32 Monte Carlo states per reference
and reports the mean and standard error over references.

The signal-space row and null dispersions are computed as
\[
D_{\rm row}
=\frac{1}{r_A}\Tr\Cov(P_A\hat x_0),
\qquad
D_{\rm null}
=\frac{1}{d-r_A}\Tr\Cov(P_A^\perp\hat x_0),
\]
whenever the corresponding subspace has nonzero dimension, with
\(P_A=A^\dagger A\) and \(P_A^\perp=\Id_d-P_A\).

\runinhead{Numerical accuracy and high-noise scaling.}
The common exact-score trajectories use a log-spaced Euler--Maruyama grid.
Table~\ref{tab:app-em-stability} compares 200, 400, and 800 reverse steps while
holding references, chains, and saved checkpoints fixed.  The index half-time
varies by at most \(0.059\) within a task, and the terminal residual remains
within \(6.3\%\) of the analytic compatible-pool floor
\[
F_r=\frac{1}{d_y}\sum_iw_i(r)\|Ax_i-r\|^2.
\]

\begin{table}[H]
\centering
\caption{Euler--Maruyama stability for the exact-score reverse SDE.  Half-times
are interpolated within each reference before averaging.  The last column is
the maximum difference between the reference-averaged reverse and forward
conditional index-entropy curves.}
\label{tab:app-em-stability}
\footnotesize
\setlength{\tabcolsep}{4pt}
\begin{tabular}{lcccc}
\toprule
Task (steps 200/400/800) & index $t_{1/2}$ & class $t_{1/2}$ & $R_{\rm end}/F_r$ & entropy gap \\
\midrule
Denoising & 0.848/0.906/0.871 & 0.887/0.998/0.992 & 0.997/1.062/0.975 & 0.094/0.074/0.074 \\
Super-resolution & 0.915/0.921/0.922 & 0.868/0.861/0.870 & 0.992/1.001/1.024 & 0.136/0.091/0.065 \\
Inpainting & 0.891/0.893/0.906 & 0.868/0.868/0.890 & 1.012/1.029/1.016 & 0.206/0.115/0.069 \\
\bottomrule
\end{tabular}
\end{table}

The joint-law calculation samples
\(I\sim\operatorname{Unif}\{1,\ldots,n\}\), \(X_0=x_I\),
\(R=AX_0+\varepsilon\), and
\(X_t=a_tX_0+\sqrt{\Delta_t}Z\).  Exact prior and posterior scores are
evaluated for 1{,}024 joint samples at eight high-noise checkpoints.  Regressing
log energy on \(\log a_t^2\) gives force-energy slopes
\(1.003\), \(1.006\), and \(1.008\), and relative-energy slopes
\(1.002\), \(1.005\), and \(1.007\), for denoising,
super-resolution, and inpainting.  The prior-score energy per dimension stays
between \(0.996\) and \(1.014\).  The estimated leading coefficient agrees
with \(d^{-1}\Tr\Cov(\E[X_0\mid R])\) to relative error below
\(4\times10^{-4}\) in every task.

\begin{figure}[H]
\centering
\includegraphics[width=0.93\linewidth]{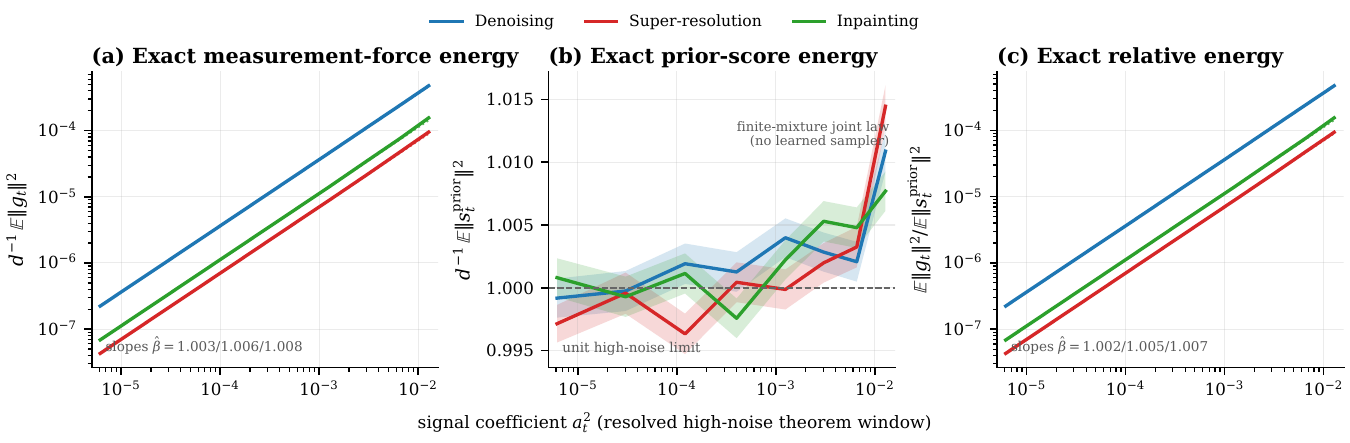}
\caption{Joint-law high-noise calculation for the finite empirical MNIST
prior (1{,}024 samples; bands are standard errors).  (a)~Measurement-force
energy per dimension.  (b)~Prior-score energy.  (c)~Their ratio.  Dotted
guides in (a) and (c) have unit slope against \(a_t^2\).}
\label{fig:joint-highnoise}
\end{figure}

\runinhead{Likelihood-width sensitivity.}
The main MNIST diagnostics use \(\sigma_y=2\).  Figure~\ref{fig:app-mnist-width}
varies the assumed likelihood width from \(0.5\) to \(3\).  As the likelihood
widens, the posterior index budgets approach \(\log n\), while the class
baselines approach their unconditional value.

\begin{figure}[H]
\centering
\includegraphics[width=0.88\linewidth]{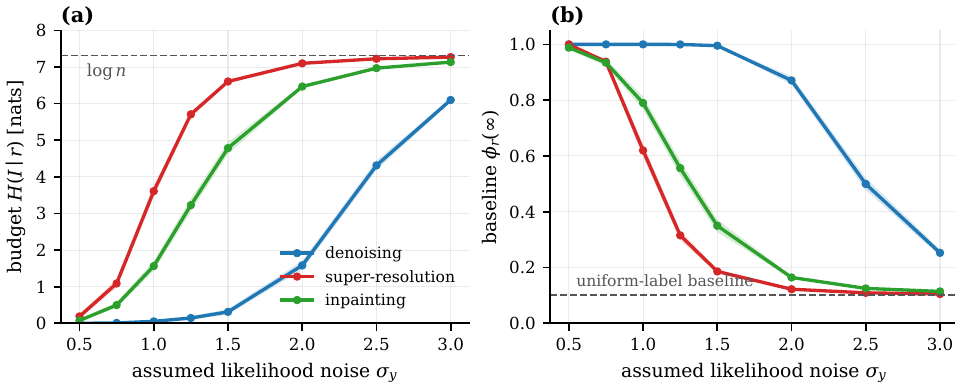}
\caption{Sensitivity of the finite-empirical-posterior diagnostics to the
assumed likelihood width.  (a)~Posterior index budget \(H(I\mid r)\).
(b)~High-noise class baseline \(\phi_r(\infty)\).  Curves average 100
in-support references; shaded bands are standard errors.  Dashed lines mark
the corresponding unconditional limits.}
\label{fig:app-mnist-width}
\end{figure}

\runinhead{Gaussian timing calibration.}
For \(2\times\) super-resolution, the fractional Gaussian resolution scale of
Remark~\ref{cor:collapse-time} is evaluated on 96
reference--likelihood-width pairs.  Its mean absolute error is
\(\HeuristicMAELogDet\), compared with \(\HeuristicMAEConstantLOSO\) for a
cross-validated constant predictor.  The paired difference
\(\HeuristicDiffConstantLOSO\) has a reference-cluster 95\% interval from
\(\HeuristicDiffConstantLOSOLo\) to \(\HeuristicDiffConstantLOSOHi\).
Thus the Gaussian scale provides a parameter-free descriptive calibration,
while its advantage over the constant is unresolved at this sample size.

\FloatBarrier

\subsection{Neural-Sampler Experiments}
\label{sec:app-neural-protocol}

\runinhead{Implementations and data split.}
The experiments use the official DPS residual-gradient update
\citep{chung2023diffusion}, the SVD-based DDNM/DDNM+ update
\citep{wang2023zero}, the \PiGDM\ pseudo-inverse rule
\citep{song2023pigdm}, and the DMPS linear-Gaussian correction
\citep{meng2022diffusion}.  Average-pooling super-resolution and coordinate
inpainting are implemented with explicit pseudo-inverses.  Projector tests give
a maximum row/null reconstruction error of \(1.192\times10^{-7}\) and a
maximum relative trace-decomposition gap of \(2.733\times10^{-10}\).

All methods use the same unconditional FFHQ \(256\times256\) checkpoint and
measurement noise \(\sigma_y=0.05\).  Ten references are used to select the
task-specific configurations in Table~\ref{tab:sampler-hyperparams}.  The main
trajectory experiment uses 30 disjoint test references with eight chains per
reference.  The finite-chain calculation uses five of these references with
64 chains.  Confidence intervals resample references 2{,}000 times; chains and
checkpoints are treated as repeated observations within a reference.

\begin{table}[tb]
\centering
\caption{Configurations selected on the ten validation references.  Every run
uses 1000 reverse steps and \(\sigma_y=0.05\).}
\label{tab:sampler-hyperparams}
\small
\begin{tabular}{lcccc}
\toprule
Task & DPS scale & DDNM+ $\eta$ & \PiGDM\ $\lambda$ ($\eta$) & DMPS scale \\
\midrule
SR $4\times$ & 0.40 & 0.85 & 0.05 (1.0) & 1.50 \\
Box inpainting & 0.65 & 0.85 & 0.10 (1.0) & 1.75 \\
Random inpainting & 0.65 & 0.85 & 0.10 (1.0) & 1.75 \\
\bottomrule
\end{tabular}
\end{table}

\runinhead{Common observables.}
At each saved state, before method-specific correction, the shared checkpoint
produces
\[
\widehat x_{0,t}^{\rm common}=D_{\rm fixed}(x_t,t).
\]
The measurement residual is
\[
R_t^{\rm pre}=d_y^{-1}\|A\widehat x_{0,t}^{\rm common}-r\|^2.
\]
For dispersion, the same \([-1,1]\) clipping and explicit row/null projector
are applied to every method, with covariance denominator \(K-1\).  Applying
the common convention removes differences in native clipping and storage
precision.  The configuration-level comparison concerns the full sampling
rules, which differ simultaneously in guidance scale, pivot, Jacobian
treatment, projection, stochasticity, and correction point.

\runinhead{Fixed-spectrum mask library and statistic.}
The held-out experiment uses ten new FFHQ references and a library of sixteen
masks.  Five masks are contiguous, three contain multiple holes, three are
structured and distributed, and five are independently generated random masks.
Every mask hides one quarter of the spatial positions, so all masks have the
same rank, nullity, and singular values.  Measurement and ancestral noise are
paired across methods and masks for each reference and chain.

\begin{figure}[tb]
\centering
\includegraphics[width=0.82\linewidth]{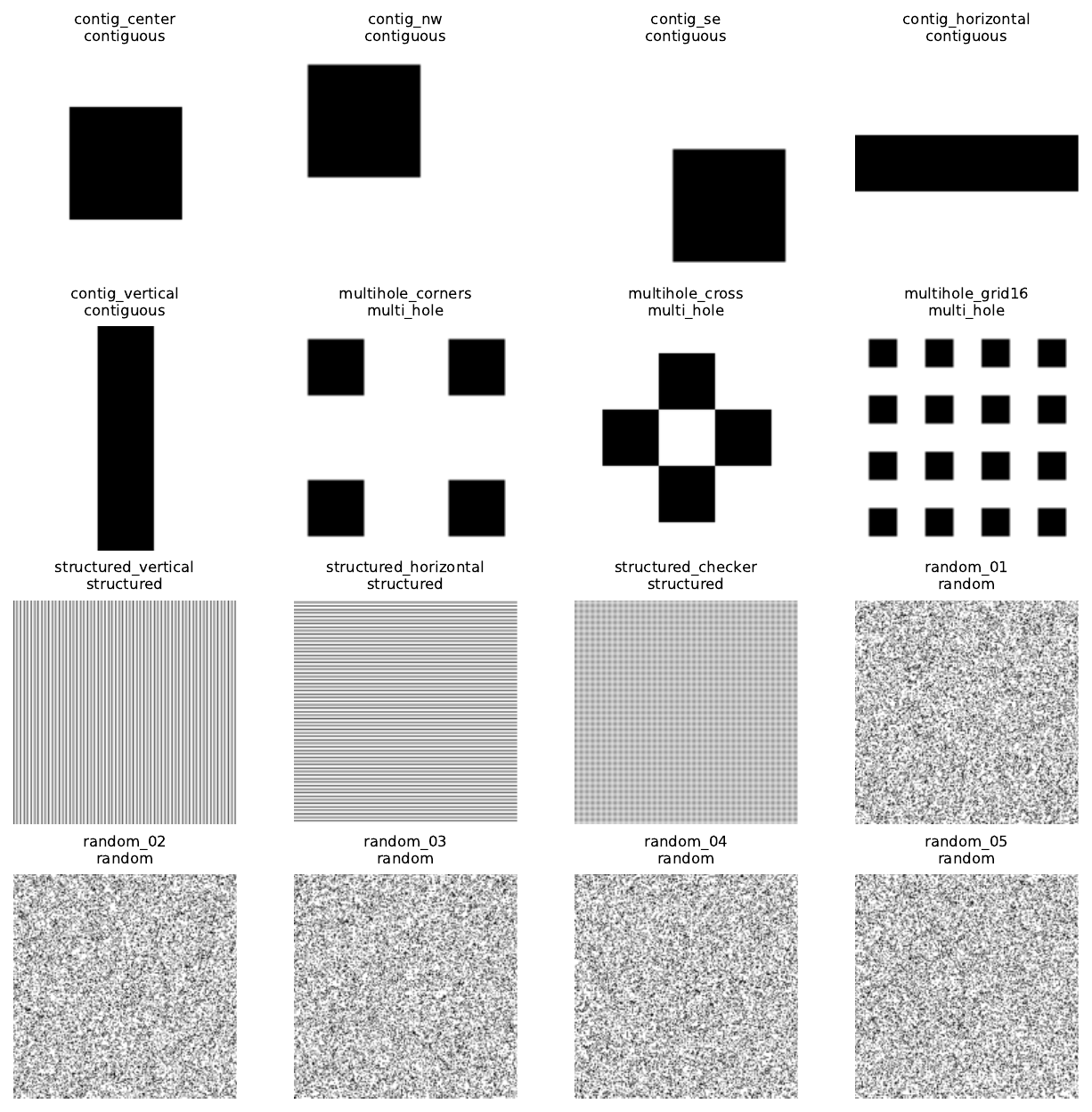}
\caption{Fixed-spectrum mask library.  Black pixels are missing.  All masks
remove the same number of spatial positions and therefore have identical
singular values.}
\label{fig:fixed-spectrum-mask-library}
\end{figure}

For method \(m\), mask \(q\), and reference \(r\), the primary statistic is
\begin{equation}
\label{eq:fixed-library-primary-cell}
L_{mqr}
=
\log\frac{\operatorname{AUC}_{\log\mathrm{SNR}}
D^{m,q,r}_{\rm null}}
{\operatorname{AUC}_{\log\mathrm{SNR}}
D^{\rm prior,q,r}_{\rm null}},
\end{equation}
where each AUC integrates the 25 common checkpoints over log-SNR
\([-10.12,9.21]\).  The denominator is obtained from an independent
unconditional DDPM trajectory pool projected through the same mask.  The
primary comparison is the contiguous-minus-random contrast computed separately
for each sampler.  The intervals condition on the fixed mask library and
resample the ten references.

\runinhead{Sensitivity analyses.}
The method-specific estimates and reference-bootstrap intervals appear in
Figure~\ref{fig:fixed-library-contrasts}.  Crossed
reference--mask bootstrap intervals give the same positive sign for
the four method-specific null-space contrasts.  Multi-hole masks give an
intermediate contrast for every sampler, whereas structured-distributed masks
are close to the random family.  Endpoint missing-region PSNR is lower for
contiguous masks than for random masks for every sampler, with method-specific
contrasts ranging from \(-23.49\) to \(-14.04\) dB.  A context-fill difficulty
adjustment has an interval spanning zero, and the two mask families have little
overlap in that covariate.  The result is therefore interpreted conditionally
on this mask library and checkpoint.

The 64-chain subset experiment measures finite-chain error in the dispersion
summaries.  Relative to 64 chains, eight chains have pooled mean absolute
relative errors of \(8.0\%\) for row-space AUC and \(7.1\%\) for null-space
AUC; the corresponding peak errors are larger.  These results support
aggregate AUC comparisons while leaving reference-level peak rankings
imprecise.  The accompanying code release supplies the sampling and analysis
code and the configurations used for these checks.

\FloatBarrier

\subsection{Exact and Plug-In Measurement Forces}
\label{sec:app-exact-vs-plugin-force}

The calibration in Section~\ref{sec:sampler-implications} can be seen in one
dimension.  Let
\[
X_0\sim\calN(0,\sigma_x^2),
\qquad
R=X_0+\varepsilon,
\qquad
X_t=a_tX_0+\sqrt{\Delta_t}\,\xi,
\]
where \(\varepsilon\sim\calN(0,\sigma_y^2)\), \(\xi\sim\calN(0,1)\), and
\(X_0\), \(\varepsilon\), and \(\xi\) are mutually independent.  Define
\[
D_t=a_t^2\sigma_x^2+\Delta_t,
\qquad
m_t(x)=\frac{a_t\sigma_x^2}{D_t}x,
\qquad
v_t=\frac{\sigma_x^2\Delta_t}{D_t}.
\]
Then \(R\mid X_t=x\sim\calN(m_t(x),v_t+\sigma_y^2)\).  The exact smoothed
likelihood force and the plug-in gradient formed with the exact prior denoiser
are
\begin{equation}
\label{eq:app-1d-forces}
g_t^{\rm exact}(x,r)
=m_t'(x)\frac{r-m_t(x)}{v_t+\sigma_y^2},
\qquad
g_t^{\rm plug}(x,r)
=m_t'(x)\frac{r-m_t(x)}{\sigma_y^2}.
\end{equation}
Consequently,
\begin{equation}
\label{eq:app-plugin-stiffness-ratio}
g_t^{\rm plug}
=\left(1+\frac{v_t}{\sigma_y^2}\right)g_t^{\rm exact},
\quad
\E|g_t^{\rm exact}|^2=\frac{(m_t')^2}{v_t+\sigma_y^2},
\quad
\E|g_t^{\rm plug}|^2
=\frac{(m_t')^2(v_t+\sigma_y^2)}{\sigma_y^4}.
\end{equation}
Thus the plug-in loss uses the measurement variance in place of the predictive
variance.  This discrepancy remains even when the clean denoiser is exact.

The discrete DPS implementation normalizes its step by the residual norm
\citep{chung2023diffusion}:
\[
\zeta_i=\frac{\zeta'}{\|y-A\hat x_0(x_i)\|}.
\]
Up to a constant absorbed into \(\widetilde\zeta\), its one-dimensional
counterpart is
\begin{equation}
\label{eq:app-dps-normalized-force}
g_t^{\rm norm}(x,r)
=\widetilde\zeta\,m_t'(x)\operatorname{sign}\{r-m_t(x)\},
\qquad
\frac{\E|g_t^{\rm norm}|^2}{\E|g_t^{\rm exact}|^2}
=\widetilde\zeta^2(v_t+\sigma_y^2).
\end{equation}
With \(\widetilde\zeta=1/\sigma_y\), which matches the low-noise energy, the
ratio is \(1+v_t/\sigma_y^2\).  Residual normalization therefore reduces the
fixed-precision mismatch, but it does not reproduce the predictive-variance
damping of the exact force.  Under the VP normalization, all three energies
have the same \(O(a_t^2)\) high-noise order; their calibration differs.